\documentclass{article}

\usepackage{arxiv}

\usepackage[utf8]{inputenc} % allow utf-8 input
\usepackage[T1]{fontenc}    % use 8-bit T1 fonts
\usepackage{hyperref}       % hyperlinks
\usepackage{url}            % simple URL typesetting
\usepackage{booktabs}       % professional-quality tables
\usepackage{amsfonts}       % blackboard math symbols
\usepackage{nicefrac}       % compact symbols for 1/2, etc.
\usepackage{microtype}      % microtypography
\usepackage{lipsum}		% Can be removed after putting your text content
\usepackage{graphicx}
\usepackage{natbib}
\usepackage{doi}
\usepackage{amsmath}
\DeclareMathOperator*{\argmin}{arg\,min}
\usepackage{algorithm}
\usepackage{algorithmic}
\usepackage{graphicx}
\usepackage{amsmath} 
\usepackage{amssymb}  % For blackboard bold
\usepackage{amsthm}
\newtheorem{assumption}{Assumption}
\newtheorem{remark}{Remark}
\newtheorem{lemma}{Lemma} 
\newtheorem{theorem}{Theorem}
\usepackage{subfig}
\usepackage{tablefootnote} %for table footnote

\title{Deep SOR Minimax Q-learning for Two-player Zero-sum Game}

\author{ 
    Saksham Gautam 
    \\
    Department of Computer Science and Automation\\
    Indian Institute of Science, Bengaluru\\
    India \\
    \texttt{gautamsaksham103@gmail.com} \\
    \And
    Lakshmi Mandal \\
    Department of Computer Science and Automation\\
    Indian Institute of Science, Bengaluru\\
    India \\
    \texttt{lakshmi.mandal.cse@gmail.com} \\
    \And
    Shalabh Bhatnagar \\
    Department of Computer Science and Automation\\
    Indian Institute of Science, Bengaluru\\
    India \\
    \texttt{shalabh@iisc.ac.in} \\
}
\renewcommand{\shorttitle}{\textit{arXiv} Template}

\hypersetup{
pdftitle={A template for the arxiv style},
pdfsubject={q-bio.NC, q-bio.QM},
pdfauthor={David S.~Hippocampus, Elias D.~Striatum},
pdfkeywords={First keyword, Second keyword, More},
}

\begin{document}
\maketitle

\begin{abstract}
	In this work, we consider the problem of a two-player zero-sum game. In the literature, the successive over-relaxation Q-learning algorithm has been developed and implemented, and it is seen to result in a lower contraction factor for the associated Q-Bellman operator resulting in a faster value iteration-based procedure. However, this has been presented only for the tabular case and not for the setting with function approximation that typically caters to real-world high-dimensional state-action spaces. Furthermore, such settings in the case of two-player zero-sum games have not been considered. We thus propose a deep successive over-relaxation minimax Q-learning algorithm that incorporates deep neural networks as function approximators and is suitable for high-dimensional spaces. We prove the finite-time convergence of the proposed algorithm. Through numerical experiments, we show the effectiveness of the proposed method over the existing Q-learning algorithm. Our ablation studies demonstrate the effect of different values of the crucial successive over-relaxation parameter. 
\end{abstract}

% keywords can be removed
\keywords{Reinforcement Learning \and Two-Player Zero-Sum Game \and Q-Learning}

\section{Introduction}
Markov Decision Processes (MDPs) provide a foundational framework for sequential decision-making, wherein a decision-maker aims to identify an optimal policy or sequence of actions that minimizes cumulative cost or maximizes cumulative reward \cite{puterman2014markov}, \cite{bertsekas2012dynamic}.  Reinforcement Learning (RL) offers a class of model-free algorithms designed to solve MDPs when the underlying system dynamics are unknown \cite{sutton2018reinforcement}, \cite{bertsekas2012dynamic}. In recent years, single-agent RL has seen significant advancements, with applications spanning diverse domains such as natural language processing and robotic control \cite{tang2025deep}.

However, many real-world scenarios involve multiple agents interacting within the same environment, influencing each other’s decisions and the system's evolution. Such settings often involve a mixture of cooperation and competition and are naturally modeled as Markov Games (MGs), also referred to as stochastic games in earlier literature \cite{shapley1953stochastic}. MGs generalize MDPs to multi-agent contexts, where each agent seeks to optimize its own objective—either jointly or independently. A key subclass includes general-sum games \cite{sahabandu2024rl}, where agents' rewards are not strictly opposed, and zero-sum games, where one agent’s gain is exactly the other’s loss \cite{chang2010adaptive}. In this work, we consider two-player zero-sum games as these are crucial real-world problems, less explored, and challenging to solve. 
In the literature, Q-learning-based approaches are studied to solve this problem  \cite{zhu2020online, jeongfinite}. However, in the case of single-agent as well as for the two-player zero-sum game, it has been seen that the convergence of standard Q-learning (QL) algorithms is slow \cite{Kamanchi2020,diddigi2022generalized}. 

To mitigate the issue of slow convergence in Q-learning the concept of successive relaxation has been previously explored in \cite{Kamanchi2020} and later \cite{john2020deep} proposed Successive Over-Relaxation (SOR) Q-learning for model-free single-agent MDPs with function approximation. 
% The concept of successive relaxation in the context of MDPs has been previously explored. The idea was first introduced in MDPs in \cite{Kamanchi2020}, and later \cite{john2020deep} proposed Successive Over-Relaxation (SOR) Q-learning for model-free single-agent MDPs with function approximation. 
This technique is further extended to the two-player zero-sum game in the tabular settings by \cite{diddigi2022generalized}, where asymptotic convergence guarantees are established using a tabular Q-value representation. However, such tabular methods are limited by the curse of dimensionality, making them impractical in large or continuous state-action spaces. Furthermore, the finite-time convergence guarantee of their work is not provided in \cite{diddigi2022generalized}.

In this paper, we propose a model-free, online Generalized Minimax Q-learning algorithm that incorporates successive over-relaxation to speed up the learning process and is applicable to function approximation settings. We provide a finite-time analysis of the algorithm and demonstrate its effectiveness when using neural networks as function approximators. For theoretical tractability, we present convergence analysis under linear function approximation. Our empirical results across a variety of environments show that the proposed algorithm achieves faster convergence compared to the standard minimax Q-learning approach. We now summarize our contributions as follows:\\
% \begin{itemize}
    1. We propose a model-free data-driven deep SOR Minimax Q-Learning (D-SOR-MQL) in the competitive multi-agent RL settings, that is, for the two-player zero-sum game.\\ 
    2. Our proposed algorithm provides faster convergence than the existing QL for the two-player zero-sum game.\\
    3. We provide a finite-time complexity of the proposed algorithm as $O(\epsilon^{-2})$, which is as good as in QL.\\
    4. Further, our experimental results on different competitive multi-agent RL environments confirm the effectiveness of our algorithm over standard QL.

\section{Related Works}
The work most closely related to ours is that of \cite{zhu2020online}, which proposes an online minimax Q-network learning algorithm that trains a neural network using observed transitions. Their approach employs generalized policy iteration to formulate a double-loop iterative scheme for computing Nash equilibria. While the paper provides online convergence guarantees in the tabular setting, it does not investigate the finite-time convergence properties of the algorithm. Moreover, it does not incorporate Successive Over-Relaxation.

The finite-time analysis of minimax Q-learning remains largely underexplored in the literature. The only work that explicitly addresses this problem is \cite{jeongfinite}, which adopts a control-theoretic perspective to analyze the learning dynamics. Specifically, the authors model the update process as a switching system and construct upper and lower comparison systems to derive finite-time bounds for both the standard minimax Q-learning algorithm and its smooth variant. While this provides valuable theoretical insights, the analysis is restricted to the tabular setting and does not extend to scenarios where the Q-function is approximated using function approximators, which are essential for scaling to high-dimensional state-action spaces.

Other studies, such as \cite{xu2020finite} and \cite{chen2019performance}, provide convergence guarantees for standard Q-learning under a tight assumption between the sampling policy and the greedy policy. While these results can, in principle, be extended to the minimax setting, the underlying assumption becomes overly restrictive in our context, particularly due to the presence of the Successive Over-Relaxation parameter \( w \geq 1 \).

The work \cite{qu2020finite} provides a finite-time analysis framework for asynchronous nonlinear stochastic approximation algorithms under a weighted infinity-norm contraction. This method introduces a recursive decomposition of the error, yielding a transparent view of how stochastic noise impacts the learning process. Although this analysis is not directly applicable to our problem setting, it inspires the structure of our theoretical analysis.

When the space of state-action pairs grows significantly, traditional Q-learning may become computationally infeasible because of the curse of dimensionality. To address this challenge, function approximation techniques are commonly employed to estimate the optimal Q-value function $Q^*$. In this work, we present our simulation results using Deep Neural Networks (DNN) to approximate the $Q^*$. However, for theoretical analysis, we adopt a linear function approximation framework to facilitate tractable analysis.
In our analysis, we approximate the Q-value function using linear function approximator. We begin by formulating a recursive update equation for the trainable network parameter $\theta$. Our decomposition closely mirrors that of the aforementioned work, allowing us to analyze the influence of stochasticity in a structured manner. However, the specific convergence analysis diverges significantly due to the distinct characteristics introduced by the minimax structure and the successive over-relaxation mechanism. Moreover, unlike previous work, which does not consider function approximation, we establish our results within the function approximation setting.

\section{Background}
\label{sec_background}
Standard Markov Decision Processes(MDPs) are defined by a tuple ${(\mathcal{S},\mathcal{A},\mathcal{P},\mathrm{R},\gamma)}$, where $\mathcal{S}$ is the state space of the agent, $\mathcal{A}$ is the action set, $\mathcal{P}$ is the transition probability, $\mathrm{R}$ is the reward function, and $\gamma \in (0,1)$ is the discount factor. For a single agent, the subsequent state $s_{t+1}$ and the reward received $r_{t+1}$ depend on the actions of the agent, specifically $s_{t+1}\sim \mathcal{P}(.\mid s_t,a_t)$ and $r_{t+1}=\mathrm{R}(s_t,a_t)$.

In contrast, this study addresses Two-Player Zero-Sum Markov Games, which are formally defined by a six-element tuple ${(\mathcal{S},\mathcal{A},\mathcal{O},\mathcal{P},\mathrm{R},\gamma)}$\cite{lagoudakis2012value}. Within this framework $\mathcal{S}$ is the game's state space, $\mathcal{A}$ represents the action set of one player, $\mathcal{O}$ denotes the action set for the opposing player, $\mathcal{P}$ denotes the state transition function, $\mathrm{R}$ specifies the reward function and $\gamma \in (0,1)$ serves as the discount factor. Our focus in this paper is on the simultaneous action games where both players determine and execute their actions concurrently during each step of play. The subsequent state $s_{t+1}$ is governed by the distribution $\mathcal{P}(s_{t+1}\mid s_{t},a_t,o_t)$ and the environment provides a reward signal $r_{t+1}=\mathrm{R}(s_t,a_t,o_t)$ as a feedback.
The primary objectives for the two agents in such a Markov game is to learn their optimal policies: $\pi : \mathcal{S} \rightarrow \Delta ^{\mid \mathcal{A}\mid}$ for the first agent and $\mu: \mathcal{S} \rightarrow \Delta ^{\mid \mathcal{O}\mid}$ for the second. Here $\Delta^d$ is the probability simplex in $\mathbb{R}^d$ and $\pi$ and $\mu$ denote the probability distributions over actions to be taken by the Agent 1 and Agent 2, respectively. Note that in the subsequent sections, $\|\cdot\|$ denotes the max-norm.

Let $Q^*$, the optimal value function, be defined as \cite{diddigi2022generalized} 
\begin{align}
    Q^*(s,a,o)&=w[\mathrm{R}(s,a,o)+ \gamma \sum_{s'\in \mathcal{S}}\mathcal{P}(s'|s,a,o)V^*(s')] +(1-w)V^*(s)
\end{align}
where the parameter $ 1 \leq w\leq w^* $ is the successive over-relaxation parameter\cite{diddigi2022generalized} \cite{Kamanchi2020}$)$.
Then $V^*(s)$, the expected return at the Nash equilibrium, is the unique fixed point of the above equation and satisfies 
    \begin{align}
        V^*(s)&=\max_{\pi(s)\in \Delta^{\mid \mathcal{A}\mid}} \min_{\mu(s) \in \Delta^{\mid \mathcal{O}\mid}} \sum_{o\in \mathcal{O}}\sum_{a \in \mathcal{A}} \pi(a|s) \mu(o|s)Q^*(s,a,o)\nonumber
    \end{align}
For brevity, we define an operator `$\text{val}$' and rewrite the above equation 
    \begin{align}\label{eq_value_fun}
        V^*(s)=\text{val}(Q^*(s,a,o))  
    \end{align}

The Q-learning can be viewed as a stochastic approximation (SA) algorithm aimed at solving the Bellman equation. Specifically, given a sample trajectory ${(s_k,a_k,o_k)}$ generated by sampling policy, the iterative approximate updates of $Q^*$ are given by $Q_k$ as 
\begin{align}
Q_{k+1}(s_k,a_k,o_k) &= Q_k(s_k,a_k,o_k)+ \alpha_k\Bigg[w\Big[\mathrm{R}(s,a,o)  +  \gamma \text{val}(Q_k(s_{k+1}, \cdot,\cdot))\Big] +\nonumber \\ &\quad \quad \quad (1-w)\text{val}(Q_k(s_k,\cdot,\cdot)) 
- Q_k(s_k,a_k,o_k)\Bigg]       
\end{align}
where $\alpha_k$, $k \geq 0$ is the step-size sequence. The sequence of Q-values \(\{Q_k\}\) generated by the update rule is guaranteed to converge to the optimal Q-function \(Q^*\) almost surely, provided that every state-action pair is visited infinitely often under the sampling policy and the step size satisfies standard diminishing conditions. The convergence of the modified Q-learning iterates \(\{Q_k\}\) to a Nash equilibrium has been established in~\cite{diddigi2022generalized}.

Equation \eqref{eq_value_fun} represents the general case where both players use stochastic policies, with the minimax operator optimizing over all possible probability distributions across their respective action spaces. This formulation is necessary because certain games have only mixed-strategy Nash equilibria\cite{myerson2013game}. Deterministic policies can be viewed as a special case within this broader framework. However, this expression can be simplified by restricting the opponent's strategy to a deterministic action. As shown in \cite{zhao2014mec},\cite{zhu2019invariant}, once one player’s policy is fixed, the two-player zero-sum Markov game reduces to a single-agent Markov Decision Process (MDP), in which a deterministic policy is sufficient to achieve optimality. Thus the equation reduces to 
\begin{align}
V^*(s)&=\max_{\pi(s)\in \Delta^{\mid \mathcal{A}\mid}} \min_{o \in \mathcal{O}}\sum_{a \in \mathcal{A}} \pi(a|s) Q^*(s,a,o)
\end{align}
Consequently, the Bellman minimax equation can be expressed as:
\begin{align}
    Q^*(s,a,o)&=w[\mathrm{R}(s,a,o)+ \gamma \sum_{s'\in \mathcal{S}}\mathcal{P}(s'|s,a,o)\max_\pi \min_{o'} \sum_{a'}\pi(a'\mid s')Q^*(s',a',o')] \nonumber \\ 
    &\quad  \quad  +(1-w)\max_\pi\min_{o''}\sum_{a''}\pi(a''\mid s)Q^*(s,a'',o'')
\end{align}
Upon obtaining $Q^*$, the Nash equilibrium policy $\pi^*$ is then determined by
\begin{align}
    \pi^*(s)=\arg \max_{\pi} \min_o \sum_a\pi(a\mid s)Q^*(s,a,o)
\end{align}
This optimization problem can subsequently be cast as a linear program:
\[
\left\{
\begin{array}{ll}
\max & \zeta \\
\text{s.t.} & \sum_{a} \rho(a) Q^*(s, a, o) \geq \zeta \quad \forall o \in \mathcal{O} \\
& \sum_{a} \rho(a) = 1 \\
& \rho(a) \geq 0 \quad \forall a \in \mathcal{A}
\end{array}
\right.
\]
For brevity, we define an operator $\mathcal{K}$ such that $\pi = \mathcal{K}(Q)$. This operator returns the optimal solution corresponding to the linear programming formulation described above.

The optimal solution, denoted by $\rho(a)$, corresponds to a probability distribution over the action set of the max player at a given state $s$.. The optimal objective value, $\zeta$, corresponds to $\max_\pi \min_o \sum_a \pi(a|s)Q^* (s, a, o)$. In the following sections, we introduce our proposed algorithm, where we implement deep neural networks to approximate the SOR Q-value function, unlike the existing literature \cite{diddigi2022generalized}. Moreover, we demonstrate the finite-time convergence of the proposed algorithm and empirical effectiveness on standard multi-agent competitive environments.    

\section{PROPOSED ALGORITHM}
In this section, we briefly describe our algorithm. Traditional tabular SOR Q-learning methods for two-player zero-sum game suffer from the curse of dimensionality, rendering them unsuitable for large or continuous state-action spaces. To address this, our approach employs a deep neural network to approximate the joint Q-value function. We parametrize the Q-function as \( q(s, a, o \mid \theta) \), where \(\theta\) denotes the trainable network parameters. The network takes the game state as input and outputs Q-values corresponding to each action pair \((a, o)\).

To enhance training stability and data efficiency, we incorporate a replay buffer, denoted by \(\mathcal{D}\), along with a target network parameterized by \(\theta_{\text{target}}\). An additional set of parameters, $\theta_{eval}$, is used for policy evaluation as done in \cite{zhu2020online}, and is further detailed later in this section.
We first define $T_{w}^{\pi}$ which we have subsequently used to explain our algorithm. 
\begin{align}
    [T_{w}^{\pi}(Q)](s,a,o)=& w[\mathrm{R}(s,a,o) + \gamma\sum_{s'}\mathcal{P}(s'\mid s,a,o) \min_{o'}\sum_{a'}\pi(a'\mid s')Q(s',a',o')] \nonumber \\ & 
    \quad \quad  \quad \quad + (1-w)\min_{o''}\sum_{a''}\pi(a'' \mid s)Q(s,a'',o'')
\end{align}

This equation can be interpreted as the Bellman optimality equation for the minimizer in a two-player setting. It evaluates the Q-value under the agent’s own policy \(\pi\), while taking the worst-case (i.e., minimum) expected value over the opponent’s action space at the next step.

Our algorithm builds on \textit{Generalized Policy Iteration (GPI)}~\cite{perolat2015approximate}, which unifies value iteration and policy iteration within a common framework. GPI proceeds in two alternating steps:
\begin{align}
    \pi_{i+1} = \mathcal{K}(Q_i) \quad,\quad
    Q_{i+1} = (T_{w}^{\pi_{i+1}})^n(Q_i)
\end{align}
where \(\mathcal{K}\) is the linear programming operator defined in Section \ref{sec_background}, and \(T_{w}^{\pi}\) serves as the policy evaluation operator. The latter step constitutes an \textit{optimistic policy evaluation}, where the parameter \(n\) governs the number of evaluation iterations.
Instead of achieving complete convergence to the exact value function \(Q_{\pi_{i+1}}\), this step yields an approximate value function \(Q_{i+1}\), which is subsequently used in the calculation of improved policy.

When \(n = 1\), GPI reduces to value iteration; as \(n \to \infty\), it recovers standard policy iteration. While policy iteration converges more rapidly, it is computationally more demanding. In contrast, value iteration is computationally efficient, but converges slowly.GPI provides a flexible trade-off between the two, allowing for a balance between computational efficiency and convergence speed by tuning the parameter \(n\).

    \begin{algorithm}[H]
    \footnotesize % Reduce the font size
\caption{Deep Successive Over-Relaxation Minimax Q-learning (D-SOR-MQL)}
\label{alg:m2qn2} % Changed label for uniqueness
\begin{algorithmic}[1]
\REQUIRE Initial network parameters $\theta_0$, $\theta_{\text{target}} \leftarrow \theta_0$, $\theta_{\text{eval}} \leftarrow \theta_0$,  experience replay buffer $D \leftarrow \emptyset$, target network update frequency $T$, policy evaluation iterations $n$, relaxation parameter $w > 1$
\FOR{$t = 0, 1, \dots$ until convergence} % Rephrased loop condition
    \STATE Select the agent's action using an $\epsilon$-greedy strategy:
    \STATE $a_t = \begin{cases}
    \left[\mathcal{K}\left(q(\cdot \mid \theta_t)\right)\right](s_t)) & \text{with probability } 1 - \epsilon_t \\ % Rephrased the policy line
    \text{random action from } \mathcal{A} & \text{with probability } \epsilon_t
    \end{cases}$
    \STATE Determine the opponent's action $o_t \in \mu_t(s_t)$, where $\mu_t(s_t)=\arg \min_o \sum_{a}\pi_{\text{eval}}(a \mid s_t) q(s_t,a,o |\theta_t)$(Best Response to $\pi_{eval}$)
    \STATE Execute the chosen actions $(a_t, o_t)$ in the environment, observing the subsequent reward $r_{t+1}$ and next state $s_{t+1}$
    \STATE Store the observed transition tuple $(s_t, a_t, o_t, r_{t+1}, s_{t+1})$ into the replay buffer $D$
    \STATE Sample a mini-batch of state-action transitions of size $m$ from $D$
    \FORALL{sampled transitions in the mini-batch} % Clarified loop scope
        \STATE Compute the target Q-value
         $y_k = w(r_{k+1} + \gamma\min_{o'}\sum_{a'}\pi_{\text{eval}}(a' \mid s_{k+1}) q(s_{k+1}, a', o' \mid \theta_{\text{target}})) + (1-w)\min_{o''} \sum_{a''} \pi_{\text{eval}}(a'' \mid s_{k}) q(s_{k}, a'', o'' \mid \theta_{\text{target}})$ where $\pi_{\text{eval}} = \mathcal{K}\left(q(\cdot \mid \theta_{\text{eval}})\right)$ is the current evaluation policy
    \ENDFOR
    \STATE Define mean squared error loss function: $L(\theta_t)$ Eq \eqref{eqn:loss}
    \STATE Update the primary network parameters $\theta_t$ using gradient descent:
    \\$\theta_{t+1} = \theta_t - \alpha \frac{\partial L(\theta_t)}{\partial \theta_t}$
    \IF{current timestep $t$ is a multiple of $T$} % Rephrased condition
        \STATE Synchronize the target network parameters: $\theta_{\text{target}} \leftarrow \theta_{t+1}$
    \ENDIF
    \IF{current timestep $t$ is a multiple of $nT$} % Rephrased condition, tied to inner-loop iterations
        \STATE Synchronize the evaluation network parameters: $\theta_{\text{eval}} \leftarrow \theta_{t+1}$
    \ENDIF
\ENDFOR
\end{algorithmic}
\end{algorithm}

We now describe how our algorithm utilizes Generalized Policy Iteration (GPI) in the approximate setting. Assume that at iteration \((i-1)\), we have a neural network-based Q-function parameterized by \(\theta_{i-1}\). The improved policy is then defined as
$
\pi_i(s) = [\mathcal{K}(q(\theta_{i-1}))](s),
$
where \(\mathcal{K}\) is a policy improvement operator defined in the previous section.

The subsequent optimistic policy evaluation for \(\pi_i\) starts with an initialization of \(\theta_{i,0} = \theta_{i-1}\). For each iteration $j$, \(1 \leq j \leq n\), within the inner loop, the Q-network is updated using a mini-batch of \(m\) samples \(\{(s_k, a_k, o_k)\}\) which are drawn from the replay buffer \(\mathcal{D}\). The computation of the target values and their corresponding loss function proceeds as follows:
\begin{align}
    y_k &= [\mathcal{T}_{w}^{\pi_i}(q(\theta_{i,j-1}))](s_k, a_k, o_k), \\
    L(\theta) &= \frac{1}{2m} \sum_k \left(q(s_k, a_k, o_k \mid \theta) - y_k\right)^2,\label{eqn:loss}
\end{align}
where \(\mathcal{T}_{w}^{\pi_i}\) denotes the Bellman operator for policy \(\pi_i\). We then update the network parameters by minimizing this loss using gradient descent.
Following \(n\) such inner-loop iterations, \(\theta_i \) is set to $\theta_{i,n}$ which serves as the approximate evaluation of policy \(\pi_i\) and the process advances to the next outer Generalized Policy Iteration.

In the implementation, \(\theta_{\text{eval}}\) corresponds to the parameters of \(Q_{i-1}\), and the evaluated policy is \(\pi_{\text{eval}} = \mathcal{K}[q(\theta_{\text{eval}})]\). The target network parameters \(\theta_{\text{target}}\) corresponds to the weights from the preceding inner-loop iteration \(\theta_{i,j-1}\). These parameters are utilized to calculate target values for updating the current Q-network which is defined by parameters \(\theta_t\). The parameters \(\theta_t\) are optimized through mini-batch gradient descent applied to data sampled from the replay buffer \(\mathcal{D}\). The \(\theta_{\text{target}}\) parameters are updated every \(T\) steps. Following \(n\) updates to the target network, the evaluation phase concludes and \(\theta_{\text{eval}}\) is replaced with the most recent \(\theta_{t+1}\). The algorithm then continues its execution with these newly updated networks.

\section{Finite-Time Convergence Analysis}
In this section, we present the finite-time convergence analysis of our algorithm (detailed proofs are presented in Appendix \ref{appendix}). We begin by first defining the update rule when the Q-value function is represented by linear functions. After that we define the recursion as a stochastic optimization scheme and the analysis follows.

We approximate the optimal action-value function $Q^*$ using a linear function approximator, denoted by $Q'$. Let $\psi(s,a,o) := (\psi_1(s,a,o), \ldots, \psi_d(s,a,o))^\top$ be a feature vector, where each basis function $\psi_k: \mathcal{S} \times \mathcal{A} \times \mathcal{O} \rightarrow \mathbb{R}$ for $k = 1, \ldots, d$, and $d \ll |\mathcal{S}||\mathcal{A}||\mathcal{O}|$. The approximate Q-function is then given by
\begin{align}
Q'(s,a,o) = \psi(s,a,o)^\top \theta,
\end{align}
where $\theta \in \mathbb{R}^d$ is the parameter vector.
% We use a low-dimensional approximation $\widetilde{Q}$ of $Q^*$, restricting $\widetilde{Q}$ to a linear subspace $\mathcal{Q}$ with dimension $d \ll |\mathcal{S}||\mathcal{A}||\mathcal{O}|$ similar to \cite{chen2019performance}. In particular, given a set of basis functions $\psi_\ell : \mathcal{S} \times \mathcal{A} \times  \mathcal{O}\rightarrow \mathbb{R}$, $\ell \in \{1, \ldots, d\}$, called features, the approximation of $Q^*$, parameterized by a weight vector $\theta \in \mathbb{R}^d$, is given by
% \[
% \widetilde{Q}_\theta(s, a,o) = \psi(s, a, o)^\top \theta,
% \]
% where $\psi(s, a, o) := (\psi_1(s, a, o), \ldots, \psi_d(s, a, o))^\top$. The Q-learning with linear function approximation for iteratively updating $\theta$ is then given by

The recursion for iteratively updating $\theta$ is then given by:
\begin{align}
    % \begin{align}
        \theta_{t+1}=\theta_t + \alpha_t \psi(s, a, b) \big[ w \big(\mathrm{R}(s, a, b) + \gamma \operatorname{val}(\psi(s', \cdot, \cdot)^{\top} \theta_t) \big)   +(1 - w) \operatorname{val}(\psi(s, \cdot , \cdot)^{\top} \theta_t)  - \psi(s, a, b)^{\top} \theta_t\big]
    % \end{align}
\end{align}
In the Algorithm \ref{alg:m2qn2}, we use three separate neural networks to train the parameter $\theta$, and improve training stability. However, $\theta$-update can be represented in a more simplified way and by a single recursion as follows:
% However, in this analysis, we use one recursion for better clarity and simplicity. 
% This recursion can now be represented in the following form: 
\begin{align}
        \theta_{t+1}=(1-\alpha_t)\theta_t + \alpha_t(F(\theta_t)+\mathcal{M}(t)),   
\end{align}
    where $\alpha_t$ is the learning rate, and
\begin{align}
    F(\theta_t) =\mathbb{E}\bigg[\psi(s, a, b) \big[ w \big(\mathrm{R}(s, a, b) + \gamma \operatorname{val}(\psi(s', a', b')^{\top} \theta_t) \big)  +(1 - w) \operatorname{val}(\psi(s, a'', b'')^{\top} \theta_t) \big]\bigg]
\end{align}
    and $\mathcal{M}$ is defined as the martingale difference sequence
    \begin{align}
        \mathcal{M}(t) =\psi(s, a, b) \big[ w \big(\mathrm{R}(s, a, b) + \gamma \operatorname{val}(\psi(s', a', b')^{\top} \theta_t) \big) + (1 - w) \operatorname{val}(\psi(s, a'', b'')^{\top} \theta_t) \big] - F(\theta_t).
    \end{align}
    %and is discussed in the later part of this section.

To control the gradient bias, we use a projection step similar to \cite{bhandari2018finite}. Beginning with an initial guess of $\theta_0$ such that $\|\theta_0\| \leq Z$, where $0< Z<\infty$. Then the recursion equation becomes
\begin{align}\label{eqn:projt}
    \theta_{t+1}=\Pi_{2,Z}((1-\alpha_t)\theta_t + \alpha_t(F(\theta_t)+\mathcal{M}(t)))
\end{align}
where the projection operator $\Pi_{2,Z}$ is as follows:
\begin{align}
    \Pi_{2,Z}(\theta)=\argmin_{\theta':\|\theta'\|_2 \leq Z}\|\theta - \theta'\|
\end{align}
Our proof requires the following three assumptions, which are standard in the literature. 
\begin{assumption}\label{assump:1}
Operator $F$ is $\gamma'$ contraction i.e. for any $x,y \in \mathbb{R}^\mathcal{N}$, $|| F(x)-F(y)|| \leq \gamma'||x-y||$
\end{assumption}
\begin{assumption}\label{assump:2}
    $w(t)$ is $\mathcal{F}_{t+1}$ measurable and satisfies $\mathbb{E}[\mathcal{M}(t)\mid \mathcal{F}_t]$=0. Further $\mid \mathcal{M}(t)\mid \leq {\mathcal{\bar{M}}}$
\end{assumption}
\begin{assumption}\label{assump:3}
    There exist a $\sigma \in (0,1)$ and positive integer $\tau$, such that for any $i \in \mathcal{N}$ and $t\geq \tau$, $\mathbb{P}(i_t=t\mid \mathcal{F}_{t-\tau})\geq \sigma$
\end{assumption}
Assumption~\ref{assump:1} represents a standard contraction property that has been extensively used in the literature (e.g., \cite{tsitsiklis1994asynchronous}, \cite{qu2020finite}). The parameter \(\gamma'\) denotes the contraction factor, which corresponds to the discount factor in the case of standard Q-learning in a discounted setting. In our setting, this assumption holds; however, the contraction factor is modified due to the presence of the successive over-relaxation parameter in the update rule.\\
Assumption~\ref{assump:2} pertains to the noise sequence \(\mathcal{M}(t)\), and is also a commonly adopted condition (e.g., \cite{shah2018q}).\\
Assumption~\ref{assump:3} requires that, conditioned on the history up to time \(t - \tau\), the index \(i_t\) assigns positive probability to every component \(i\). This ensures that each \(i\) is visited sufficiently often by the sampling process. This assumption is more general than many typical ergodicity conditions encountered in the stochastic approximation literature, as discussed in \cite{qu2020finite}.

% With these assumptions in place, we are now ready to state our main theoretic result.
\begin{theorem}\label{thm:1}
    Assume that Assumptions ~\ref{assump:1},~\ref{assump:2} and ~\ref{assump:3} are satisfied. Additionally suppose that there exists a constant $Z$ such that the parameter norm is bounded by $\|\theta^*\|\leq Z$ and the iterates $\theta_t$ satisfy $\|\theta_t\|\leq Z$ almost surely for all $t$. Let the step size be $\alpha_t=\frac{\mathcal{H}}{t+t_0}$ with $\max(4\mathcal{H},\tau)\leq t_0$ and $ \frac{2}{\sigma(1-
    \gamma')}\leq \mathcal{H}$. Then the following bound holds: 
\begin{align}
\resizebox{.42\textwidth}{!}{$
    \|\theta_T- \theta^*\|\leq \frac{4\widetilde{\mathcal{M}}\sqrt{\mathcal{H}\log\frac{1}{\delta}}}{(1-\delta)\sqrt{T+t_0}} + \frac{4 Z(\tau + t_0)}{(1-\gamma')(T+t_0)}$
}
\end{align}

with atleast $1-\delta$ probability
\end{theorem}
\begin{proof}

Now, from \eqref{eqn:projt} we have
\begin{equation}\label{eqn:1}
    \theta_{t+1} \leq (1-\alpha_t)\theta_t + \alpha_t(F(\theta_t)+\mathcal{M}(t))
\end{equation}
Expanding \eqref{eqn:1} recursively, we get
\begin{align}
    \theta_{t+1} \leq \prod_{h=\tau}^{t}(1-\alpha_h)\theta_\tau +\sum_{h=\tau}^{t}\alpha_h\prod_{l=h+1}^{t}(1-\alpha_l)F(\theta_h)  \nonumber +\sum_{h=\tau}^{t}\alpha_h\prod_{l=h+1}^{t}(1-\alpha_l)\mathcal{M}(h)
\end{align}
Now we present the following Lemma which decomposes the error $\|\theta_t - \theta^*\|$ in a recursive form.
\begin{lemma}\label{lemma:1}
    Let $\xi_t=\|\theta_t-\theta^*\|$, we have almost surely 
        \begin{align}\label{eqn:errde}
        \resizebox{.42\textwidth}{!}{$
   \xi_{t+1}\leq \beta_{\tau-1,t}\xi_\tau + \gamma'\sum_{h=\tau}^{t}\widetilde{\beta}_{h,t}\xi_k + \|\sum_{h=\tau}^t \widetilde{\beta}_{h,t}\mathcal{M}(h)\|$
}
        \end{align}
    here $\beta_{\tau -1, t} = \prod_{h=\tau}^t(1-\alpha_h)$, $\widetilde{\beta}_{h,t}=\alpha_h\prod_{l=h+1}^{t}(1-\alpha_l)$
\end{lemma}

To bound the terms in Lemma \ref{lemma:1}, we need to understand the behaviour of $\alpha_h$, $\beta_{h,t}$, and $\widetilde{\beta}_{h,t}$. We present the following results on these sequences, which we will use later in our proof to control the terms in \eqref{eqn:errde}
\begin{lemma}\label{lemma:2}
    If $\alpha_t=\frac{\mathcal{H}}{t+t_0}$ and $\max(4\mathcal{H},\tau)\leq t_0$ then, 
        1. $\beta_{h,t}\leq (\frac{h+1+t_0}{t+1+t_0})^\mathcal{H}$ and $\widetilde{\beta}_{h,t}=\frac{\mathcal{H}}{h+t_0}(\frac{h+1+t_0}{t+1+t_0})^\mathcal{H}$\\
        2. $\sum_{h=1}^{t}\widetilde{\beta}_{h,t}^2\leq \frac{2\mathcal{H}}{t+1+t_0}$\end{lemma}
Now, we bound $\|\sum_{h=\tau}^{t}\widetilde{\beta}_{h,t}\mathcal{M}(h)\|$ term in eq.12. We know that $\mathcal{M}(h)$ is a martingale difference sequence with $\mathbb{E}[\mathcal{M}(h)|\mathcal{F}_{h-1}]=0$. 

Let $\mathcal{M}_i(h)$ be the $i^\text{th}$ component of $\mathcal{M}(h)$ and $\epsilon >0$. Using Azuma-Hoeffding bound \cite{azuma1967weighted} for scalars we have:

\begin{align}
    \mathrm{P}_r\{|\sum_h\widetilde{\beta}_{h,t}\mathcal{M}_i(h)|\geq \epsilon\} \leq 2\exp(-\frac{\epsilon^2}{2\widetilde{\mathcal{M}}^2\sum\widetilde{\beta}^2_{h,t}}) \nonumber
\end{align}
Now we apply a union bound over $i=1,.....,d$, which gives

\begin{align}
\resizebox{.42\textwidth}{!}{$
    \mathrm{P}_r\{\|\sum_h\widetilde{\beta}_{h,t}\mathcal{M}(h)\|\geq \epsilon\}\leq 2d \exp(-\frac{\epsilon^2}{2\widetilde{\mathcal{M}}^2\sum\beta^2_{h,t}})$}
\end{align}
Solving $2d\exp(-\epsilon^2/(2\widetilde{\mathcal{M}}^2\sum\beta^2_{h,t}))=\delta$, we get
\begin{align}\label{eqn:azuma}
    \|\sum_h\widetilde{\beta}_{h,t}\mathcal{M}(k)\|&\leq \widetilde{\mathcal{M}}\sqrt{2\sum\widetilde{\beta}^2_{h,t}\ln\frac{2d}{\delta}} \nonumber\\
    &\leq 2\widetilde{\mathcal{M}}\sqrt{\frac{\mathcal{H}}{t+1+t_0}\ln(\frac{2d}{\delta})}
\end{align}
Now we can finally bound \eqref{eqn:errde}. We have to show that, with probability $1-\delta$,
\begin{align}\label{eqn:ind1}
    \xi_T\leq \frac{C_{\xi}}{\sqrt{T+t_0}} + \frac{C_{\xi}^{'}}{T+t_0}
\end{align}
where $C_{\xi}=\frac{4\widetilde{\mathcal{M}}}{1-\gamma'}\sqrt{h\log{\frac{2d}{\delta}}}$ and $C_{\xi}^{'}=\frac{4Z(\tau+t_0)}{1-\gamma'}$. Using Lemma \ref{lemma:2} and \eqref{eqn:azuma} with \eqref{eqn:errde} we have, 
\begin{align}\label{eqn:22}
    \xi_{t+1}\leq \bigg(\frac{\tau+t_0}{t+1+t_0}\bigg)^\mathcal{H}\xi_\tau + \gamma'\sum_{h=
    \tau}^{t}\widetilde{\beta}_{h,t}\xi_h + \frac{C_\mathcal{M}}{\sqrt{t+1+t_0}}
\end{align}
 with probability at least $1-\delta$, where $C_\mathcal{M} = 2\widetilde{\mathcal{M}}\sqrt{\mathcal{H}\log\frac{2d}{\delta}}$

We will use induction to show \eqref{eqn:ind1}. We first verify the base condition
    \begin{align}\label{eqn:base}
        \xi_\tau\leq\frac{C_{\xi}}{\sqrt{\tau+t_0}} + \frac{4Z}{1-\gamma'}
    \end{align}
We have $\xi_\tau\leq 2Z$ this is because $\xi_\tau=\|\theta_\tau-\theta^*\|\leq \|\theta_\tau\|+\|\theta^*\|\leq 2Z$. This comes from the projection step in our algorithm. Thus the base condition \eqref{eqn:base} holds.
Now from \eqref{eqn:22}, we have 
    \begin{align}\label{eqn:24}
        \xi_{t+1}&\leq \bigg(\frac{\tau+t_0}{t+1+t_0}\bigg)^\mathcal{H}\xi_\tau + \gamma'\sum_{h=
    \tau}^{t}\widetilde{\beta}_{h,t}[\frac{C_{\xi}}{\sqrt{t+t_0}}+\frac{C_{\xi}^{'}}{t+t_0}] + \frac{C_\mathcal{M}}{\sqrt{t+1+t_0}}\\
    &\leq \underbrace{\sqrt{\gamma'}\frac{C_{\xi}}{\sqrt{t+1+t_0}}+ \frac{C_\mathcal{M}}{\sqrt{t+1+t_0}}}_{F_t}   +\underbrace{\bigg(\frac{\tau+t_0}{t+1+t_0}\bigg)^h \xi_\tau + \sqrt{\gamma'}\frac{C_{\xi}^{'}}{t+1+t_0}}_{F'_t} \label{eqn:29}
    \end{align}
Eq\eqref{eqn:29} is due to Lemma \ref{lemma:3}, which we provide below.
\begin{lemma}\label{lemma:3}
    If $ 1\leq\mathcal{H}(1-\sqrt{\gamma'})$, $t_0\geq 1$ and $\alpha_0 \leq \frac{1}{2}$ then for any $0 < \Gamma < 1$, we have $\sum_{h=\tau}^t \widetilde{\beta}_{h,t}\frac{1}{(h+t_0)\Gamma}\leq \frac{1}{\sqrt{\gamma'}(t+1+t_0)}$.
\end{lemma}
Now to finish the proof, it's sufficient to show $F_t\leq \frac{C_{\xi}}{\sqrt{t+1+t_0}}$ and $F_t^{'}\leq \frac{C_{\xi}^{'}}{t+1+t_0}$. 

We have 
\begin{align}
    F_t\frac{\sqrt{t+1+t_0}}{C_{\xi}}&=\Bigg(\sqrt{\gamma'}\frac{C_{\xi}}{\sqrt{t+1+t_0}} +\frac{C_\mathcal{M}}{\sqrt{t+1+t_0}}\Bigg)\frac{\sqrt{t+1+t_0}}{C_{\xi}} \nonumber
    =\sqrt{\gamma'}+ \frac{C_\mathcal{M}}{C_{\xi}}, \\
    F_t^{'}\frac{t+1+t_0}{C_{\xi}^{'}}&=
    \nonumber
    \Bigg(\bigg(\frac{\tau+t_0}{t+1+t_0}\bigg)^h \xi_\tau + \sqrt{\gamma'}\frac{C_{\xi}^{'}}{t+1+t_0}\Bigg)\frac{t+1+t_0}{C_{\xi}^{'}}= \sqrt{\gamma'} + \frac{\xi_\tau(\tau+t_0)}{C_{\xi}^{'}}\frac{(\tau+t_0)^{\mathcal{H}-1}}{(t+1+t_0)^{\mathcal{H}-1}}
\end{align}
It is sufficient to prove $\frac{C_\mathcal{M}}{C_{\xi}}\leq 1 -\sqrt{\gamma'}$ and $\frac{\xi_\tau(\tau+t_0)}{C_{\xi}^{'}}\leq 1-\sqrt{\gamma'}$. These conditions satisfy as $\xi_\tau \leq 2Z$.
\end{proof}
\begin{remark}
    Our finite-time analysis obtains $O(\epsilon^{-2})$ sample complexity  (from our Theorem \ref{thm:1}), where $O(\cdot)$ denote worst case sample complexity \cite{qu2020finite}. Further, as we consider function approximation, unlike literature \cite{qu2020finite, diddigi2022generalized}, our sample complexity does not depend on the size of the state and action space (see Table \ref{tab_FA} in Appendix).
\end{remark}

\section{EXPERIMENTAL RESULTS}

\textbf{Experimental Environments:} To evaluate the proposed algorithm, we consider two distinct environments, `Guard-Invader' and `Soccer', and their respective variations with $49$ and $121$ states (randomly chosen).
The Guard-Invader game~\cite{lu2010investigation} is a non-symmetric TZMG. This game has two players named guard and invader and have different objectives. Thus, both players have to adopt inherently different strategies. The invader aims to reach the door, while the guard seeks to intercept it first. An invasion succeeds if the invader reaches the door or is caught there, whereas the guard tries to capture it as far from the door as possible.
In our experiments, we adopt the reward model proposed in \cite{zhu2020online}. From the guard’s perspective, a reward of $-10$ is assigned if the invader successfully reaches the door. If the invader is captured before reaching the door, the reward is set to the Manhattan distance between the invader's current position and the door. In all other cases, the reward is zero. We normalize the rewards to the range $[-1,1]$ before storing them in the replay buffer to improve training efficiency. We learn the guard's policy, while the opponent (invader) is controlled using the same Q-network during online learning.

%%%Guard Prisoner
\begin{figure}[htb]
\centering
\vspace{-30pt}
  \subfloat[]{
  \hspace{-20pt}
	\begin{minipage}[c][1\width]{
	   0.26\textwidth}
	   \centering
        \includegraphics[width=1\textwidth]{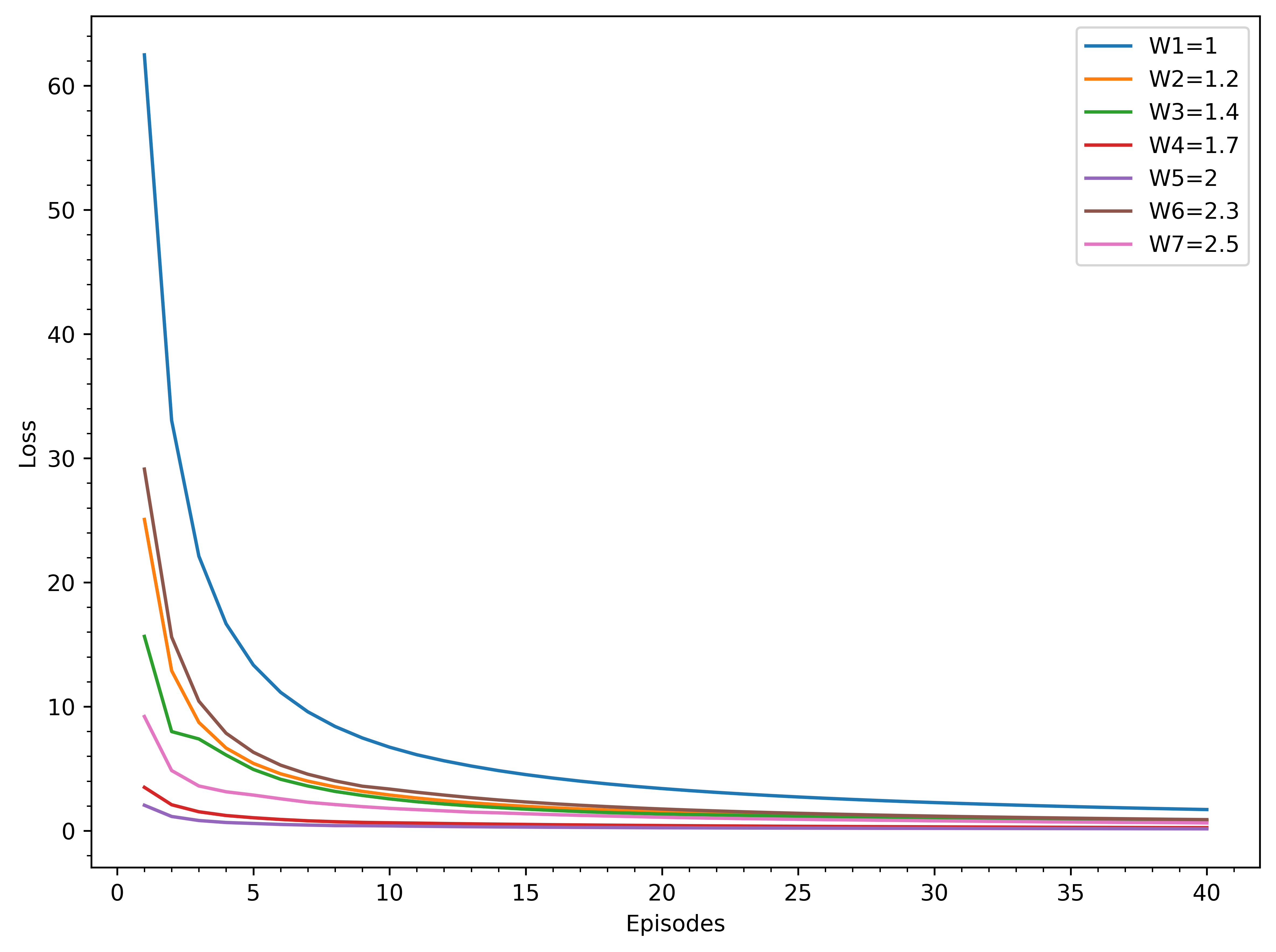}
        \vspace{-35pt}
	\end{minipage}}	
  \subfloat[]{
	\begin{minipage}[c][1\width]{
	   0.26\textwidth}
        % \hspace{-20pt}
	   \centering
	   \includegraphics[width=1\textwidth]{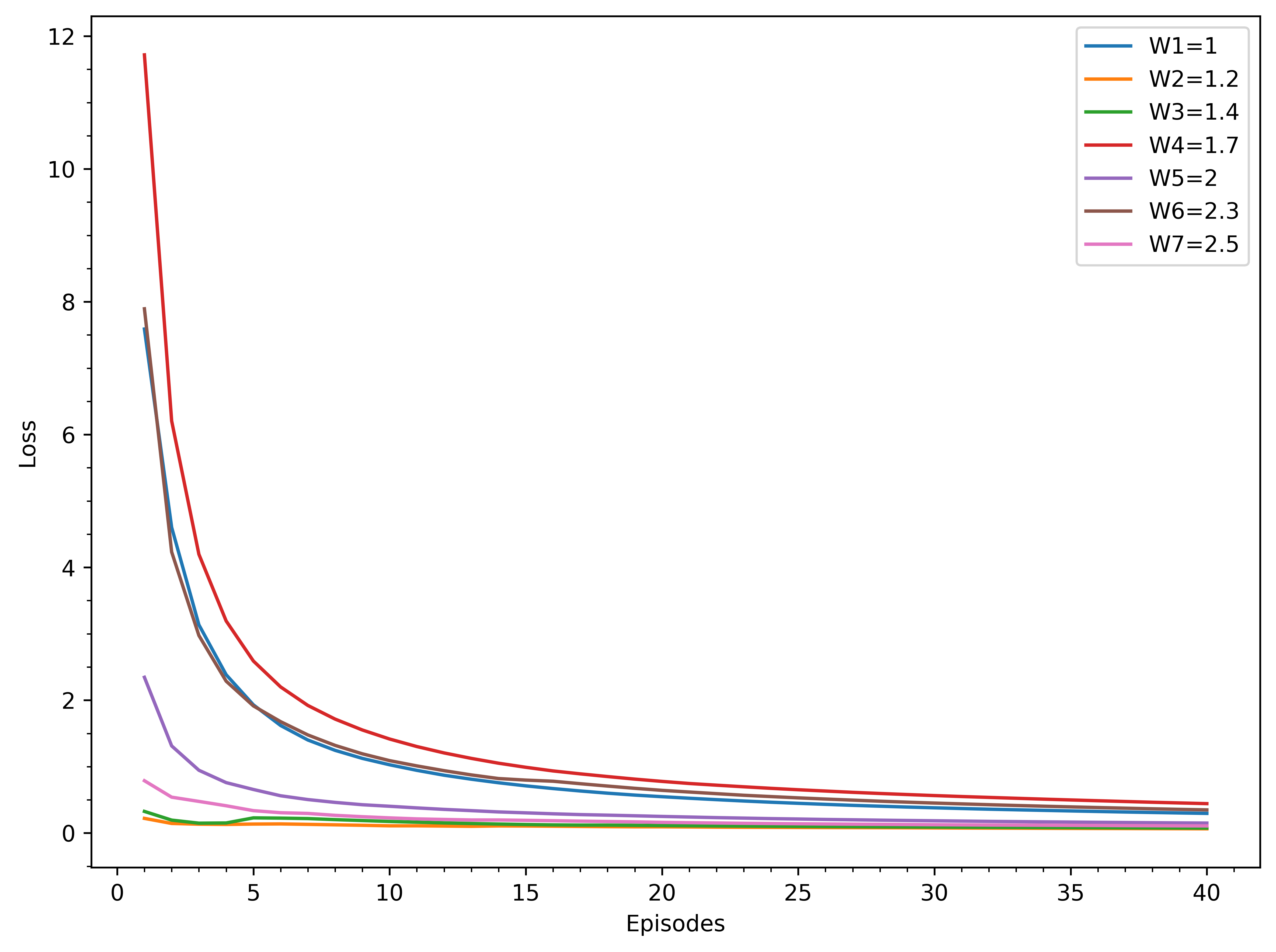}
    \vspace{-35pt}
	\end{minipage}}
 % \vspace{-10pt}
\caption{Loss on Guard-Invader environment with  $(a)$ $49$, and $(b)$ $121$ states, respectively.}
\label{fig_GI_loss}
% \vspace{-20pt}
\end{figure}

\begin{table*}[htb]
    \centering
    \scalebox{0.8}{
    \begin{tabular}{|c|c|c|c|c|c|c|}
    \hline
        Algorithm & parameter $w$ & Loss (GI-49) & Loss (GI-121) & Loss (S-49) & Loss (S-121)   \\ \hline
        M2DQN & - & $0.8617\pm 0.0405$& $0.1017\pm 0.0004$&$0.1398\pm 0.0161$&$0.1772\pm 0.0201$ \\ \hline
        D-SOR-MQL(Ours) &1.2 &$0.3003 \pm 0.00784$&$0.0501 \pm 0.0002$  &$0.0260 \pm 0.0001$     &$0.0211 \pm 1.9821e-07$\\ \hline
        D-SOR-MQL(Ours) &1.4 &$0.2824 \pm 0.0202$ &$0.0734 \pm 0.0022$  &$0.0295 \pm 5.3483e-05$ &$0.0356 \pm 0.0002$ \\ \hline
        D-SOR-MQL(Ours) &1.7 &$0.1429 \pm 0.0018 $&$0.1251 \pm 0.0030$  &$0.0356 \pm 3.1057e-05$ &$0.2519 \pm 0.0439$ \\ \hline
        D-SOR-MQL(Ours) &2.0 &$0.1507 \pm 0.0092$ &$0.1029 \pm 0.0038$  &$0.0379 \pm 0.0001$     &$0.3557 \pm 0.0919$ \\ \hline
        D-SOR-MQL(Ours) &2.3 & $0.3211 \pm 0.0117$ & $0.1399 \pm 0.0003$& $0.0469 \pm 0.0008$    & $0.2843 \pm 0.0690$ \\ \hline
        D-SOR-MQL(Ours) &2.5 &$0.2160 \pm 0.0083$ &$0.0642 \pm 0.00004$  &$0.0462 \pm 2.8592e-05$&$0.0393 \pm 5.1237e-05$ \\ \hline
        
    \end{tabular}}
    \caption{Experimental results in terms of Loss after convergence on `Guard-Invader', and ` Soccer' environments for $49$ and $121$ states, represented as `GI-49', `GI-121', `S-49', `S-121', respectively.}
    \label{tab:1}
    % \vspace{-12pt}
\end{table*}

\begin{figure}[htb]
\centering
\vspace{-30pt}
  \subfloat[]{
  \hspace{-20pt}
	\begin{minipage}[c][1\width]{
	   0.26\textwidth}
	   \centering
        \includegraphics[width=1\textwidth]{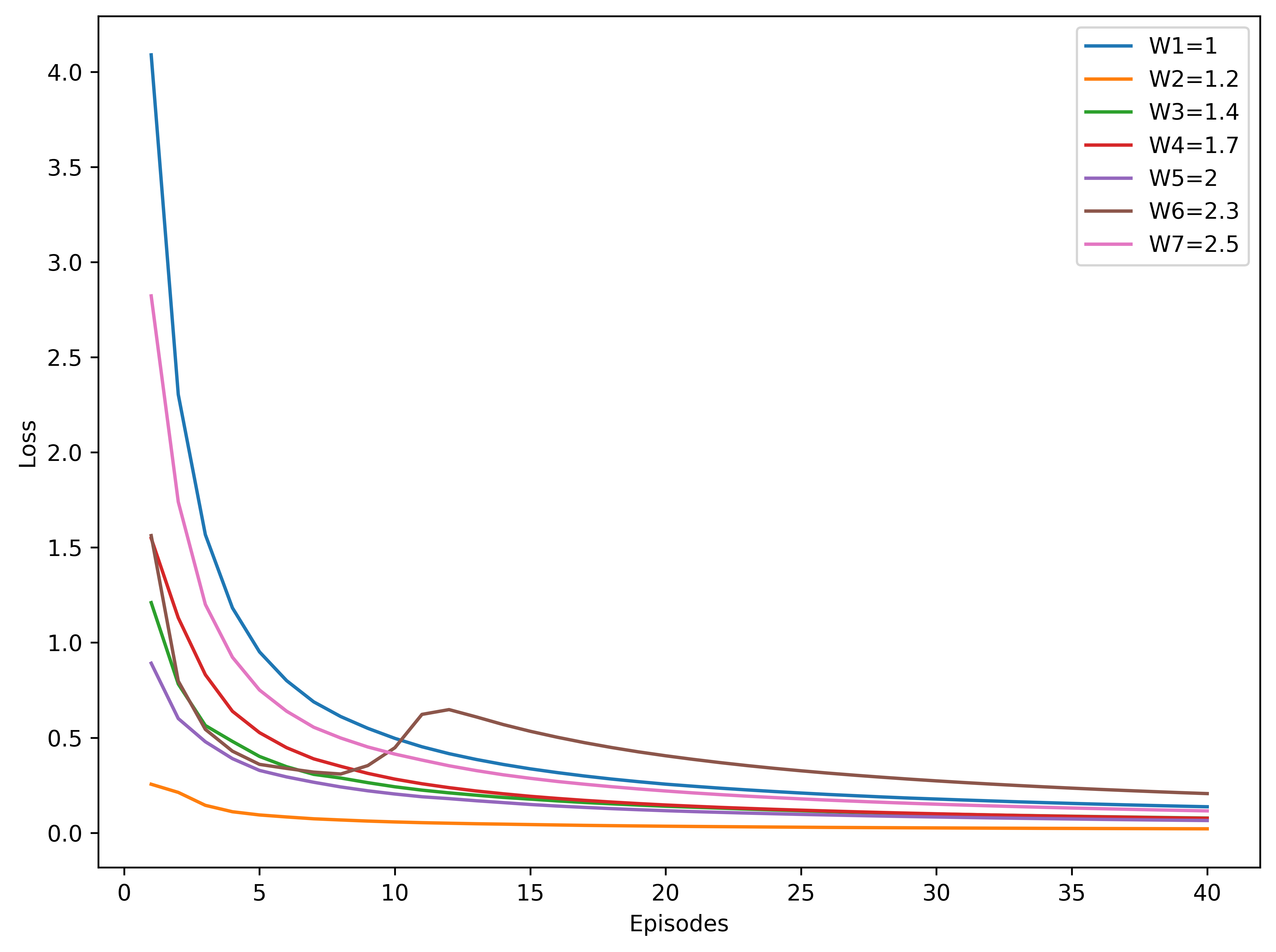}
        \vspace{-35pt}
	\end{minipage}}	
  \subfloat[]{
	\begin{minipage}[c][1\width]{
	   0.26\textwidth}
        % \hspace{-20pt}
	   \centering
	   \includegraphics[width=1\textwidth]{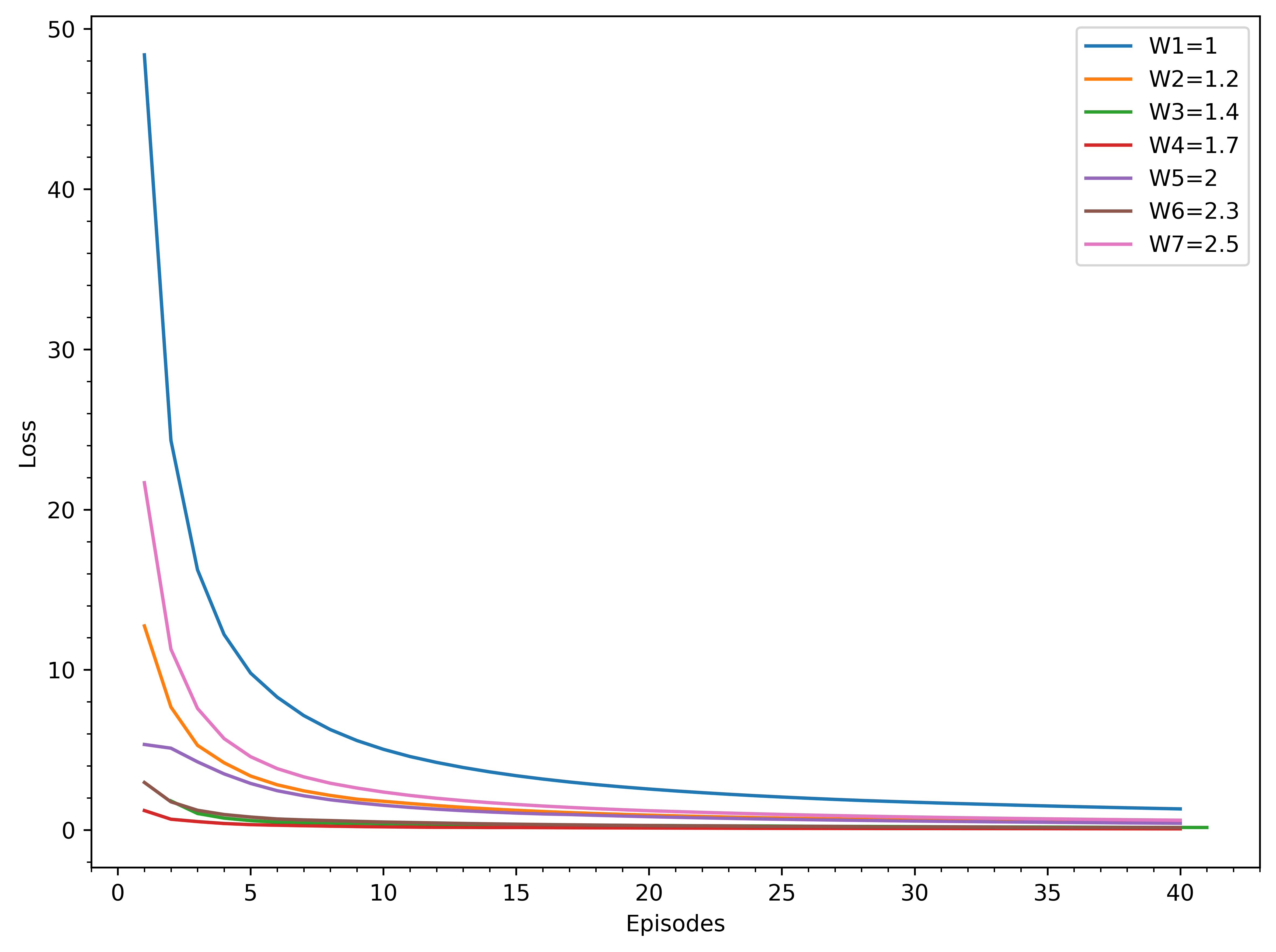}
    \vspace{-35pt}
	\end{minipage}}
 % \vspace{-10pt}
\caption{Loss on Soccer environment with  $(a)$ $49$, and $(b)$ $121$ states, respectively.}
\label{fig_Soccer_loss}
\end{figure}

\begin{figure}[htb]
\centering
\vspace{-30pt}
  \subfloat[]{
  \hspace{-20pt}
	\begin{minipage}[c][1\width]{
	   0.26\textwidth}
	   \centering
        \includegraphics[width=1\textwidth]{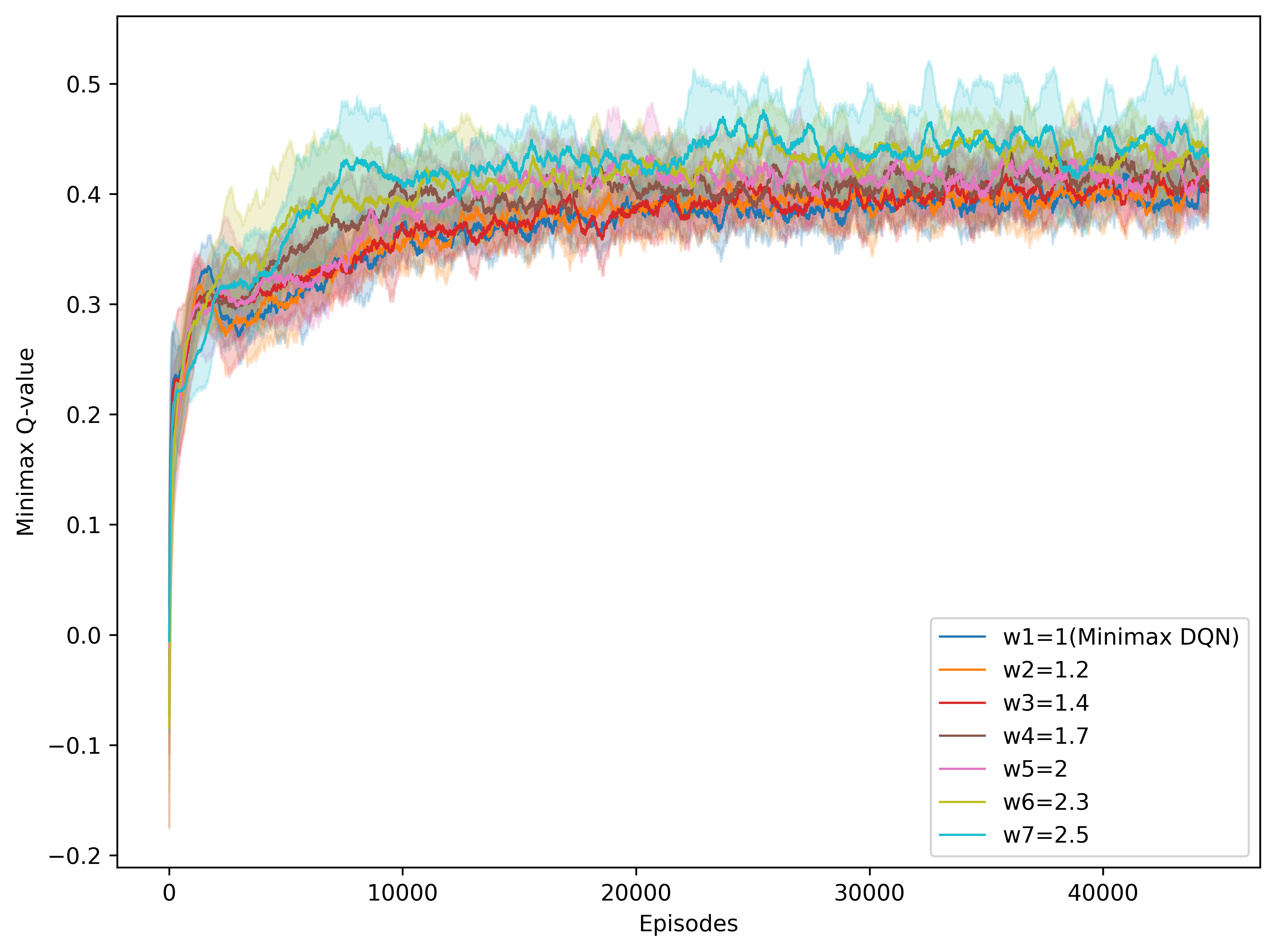}
        \vspace{-35pt}
	\end{minipage}}	
  \subfloat[]{
	\begin{minipage}[c][1\width]{
	   0.26\textwidth}
        % \hspace{-20pt}
	   \centering
	   \includegraphics[width=1\textwidth]{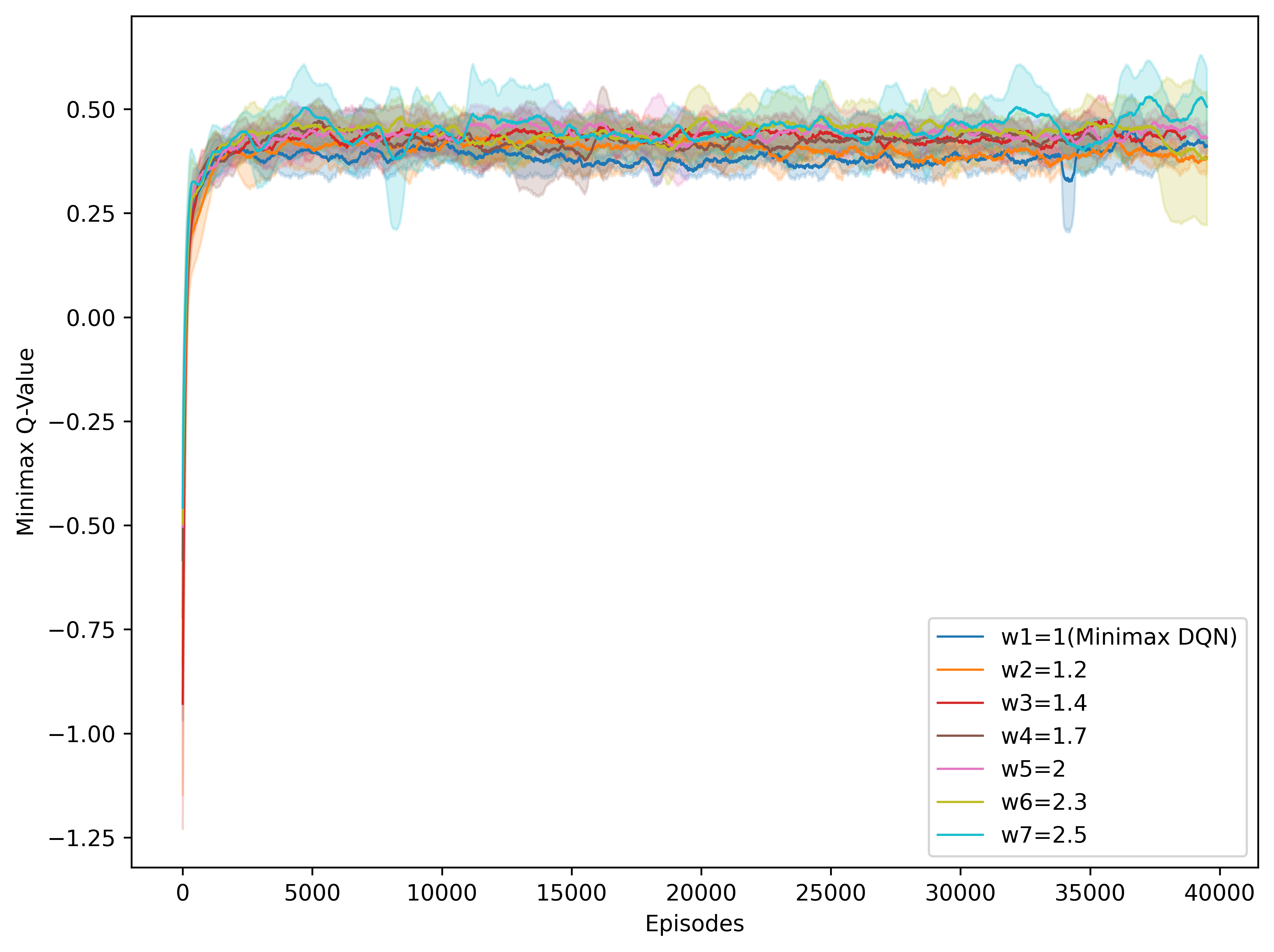}
    \vspace{-35pt}
	\end{minipage}}
 % \vspace{-10pt}
\caption{Mimimax Q-value on Guard-Invader environment with  $(a)$ $49$, and $(b)$ $121$ states, respectively.}
\label{fig_GI_Q_val}
\end{figure}

\begin{figure}[htb]
\centering
\vspace{-30pt}
  \subfloat[]{
  \hspace{-20pt}
	\begin{minipage}[c][1\width]{
	   0.26\textwidth}
	   \centering
        \includegraphics[width=1\textwidth]{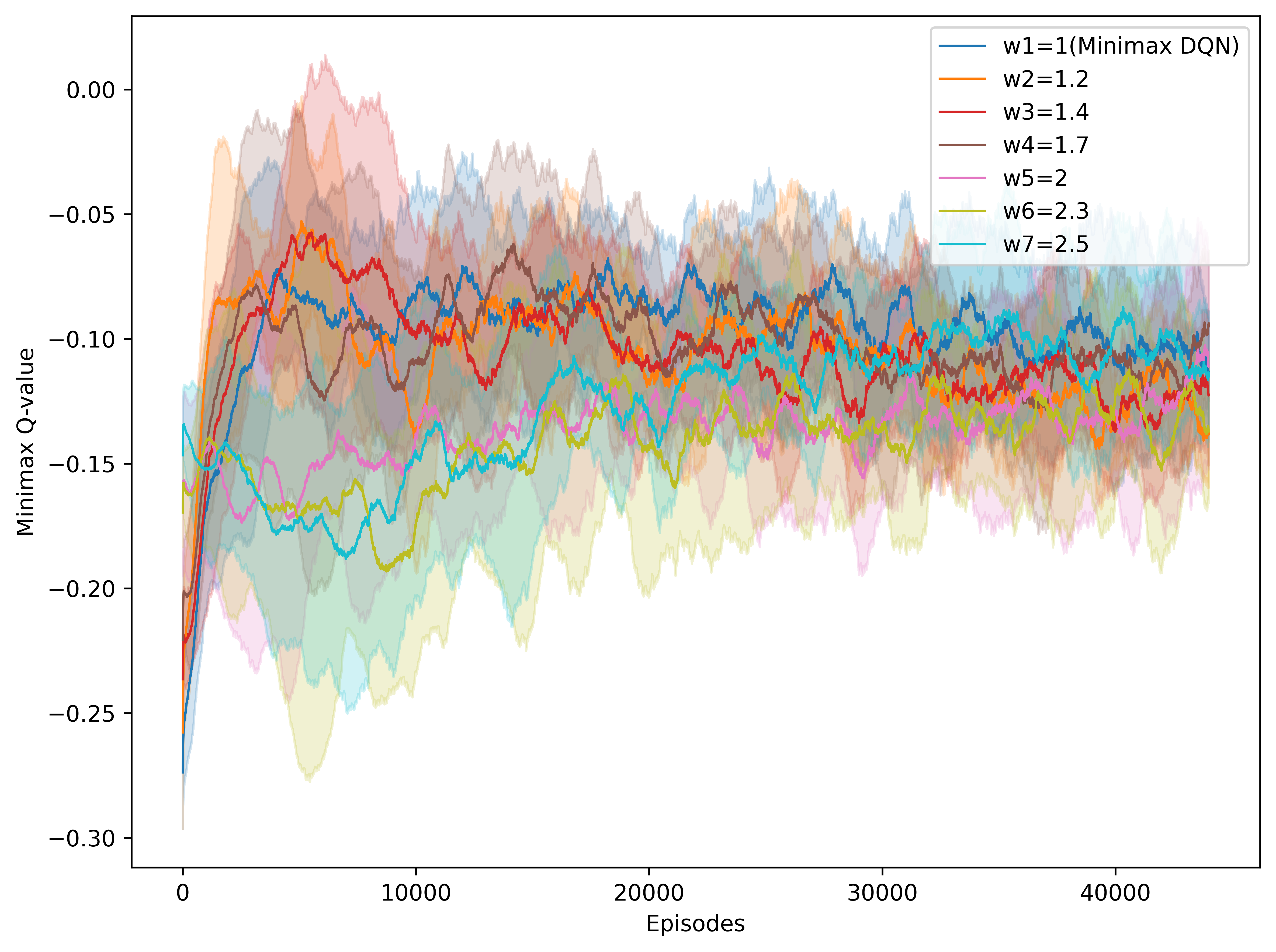}
        \vspace{-35pt}
	\end{minipage}}	
  \subfloat[]{
	\begin{minipage}[c][1\width]{
	   0.26\textwidth}
        % \hspace{-20pt}
	   \centering
	   \includegraphics[width=1\textwidth]{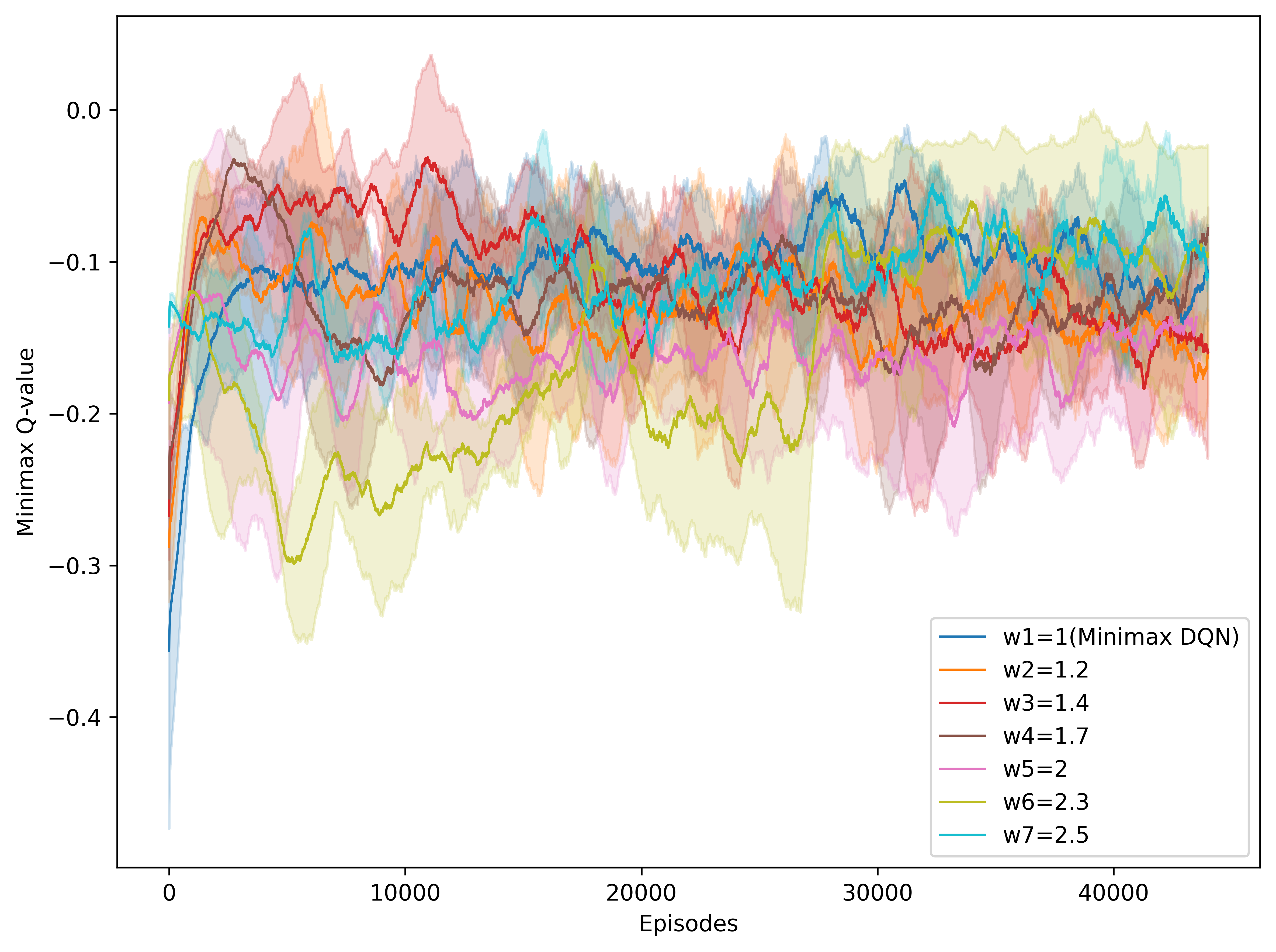}
    \vspace{-35pt}
	\end{minipage}}
 % \vspace{-10pt}
\caption{Mimimax Q-value on Soccer environment with  $(a)$ $49$, and $(b)$ $121$ states, respectively.}
\label{fig_Soccer_Q_val}
 % \vspace{-15pt}
\end{figure}

Another environment we consider is the Two-Player Soccer Game introduced in \cite{littman1994markov}. At the start of each episode, both players are randomly positioned on different cells of the grid, and the ball is assigned to one of them at random. The actions of the players are executed in a random order to ensure simultaneous play. If a player attempts to move into the opponent’s cell, the possession of the ball is transferred to the other player, and the move is aborted. A goal is scored when a player successfully moves the ball into their designated goal area, ending the game and awarding a reward of $+1$ to A (if A scores) or $-1$ to A (if B scores). In all other cases, the reward is zero. Due to the game's symmetry, a strategy learned by one player can be directly applied to the other.\\
\textbf{Parameters, Neural network architecture, Loss function:} For both the environments, we set the parameter discount factor to \(\gamma = 0.95\). The Q-network architecture consists of two hidden layers with 256 and 128 units, respectively. We used ReLU activation after each layer except for the output layer for which we used a linear activation. In the Guard–Invader game, the input includes guard and invader positions and territory coordinates, while in Soccer it includes both player's positions and a ball-possession indicator.  We use the Mean Squared Error (MSE) loss function with a batch size of 64 and a learning rate of \(5 \times 10^{-5}\) to train the network. The replay buffer has a maximum capacity of 10,000. Exploration is carried out using an \(\epsilon\)-greedy exploration policy where \(\epsilon\) decays exponentially from 1.0 to 0.1 with a decay rate of 20,000 steps. Experiments are conducted under \(T = 100\) and \(n = 5\), for various values of the parameter \(w \geq 1\). The case \(w = 1\) corresponds to the baseline Minimax Q-learning algorithm (M2QN), and serves as a point of comparison for the proposed method. The convergence rate improves with increasing $w$ up to a threshold $w^*$, after which further increase causes inconsistent performance.\\
\textbf{Results:} We present our training results in terms of loss and minimax Q-value with respect to number of training episodes, on the above-mentioned environments in Figures \ref{fig_GI_loss} to \ref{fig_Soccer_Q_val} (for a larger view, check Figures \ref{fig_GI_loss_appendix} to \ref{fig_Soccer_Q_val_appendix} in the Appendix). We further summarize the results after convergence in Table~\ref{tab:1}, where loss is presented in the form of `mean $\pm$ standard deviation'. It is observed from all figures and Table \ref{tab:1} that the convergence using our proposed algorithm for the optimal value of parameter $w$ is significantly faster than the baseline M2QN algorithm. Moreover, the mean squared error (MSE) at convergence is substantially lower compared to the baseline M2QN algorithm. The reduced variance (or standard deviation) in Table~\ref{tab:1} further indicates that our method shows greater stability.
\section{CONCLUSIONS}
In this work, we propose a deep SOR Minimax Q-Learning algorithm where deep neural networks are employed as a function approximation and resolve the issue of the curse of dimensionality that exists in the literature.
We further present a finite-time convergence analysis of our proposed algorithm with linear function approximation and demonstrate the theoretical soundness of the algorithm. Our experiments demonstrate that incorporating successive over-relaxation can accelerate convergence. In this work, while the empirical results utilize deep neural networks, the theoretical analysis is conducted using linear function approximators for tractability. As part of future work, we aim to bridge this gap by extending the analysis to settings where Q-values are approximated using deep neural networks.

\clearpage
\newpage
\bibliography{ref}

%%%%%%%%%%%%%%%%%%%%%%%%%%%%%%%%%%%%%%%%%%%%%%%%%%%%%%%%%%%%

\clearpage
\appendix
\thispagestyle{empty}

% Supplementary material: To improve readability, you must use a single-column format for the supplementary material.
\section{Appendix}
\label{appendix}
\subsection{Flowchart of the algorithm}
\begin{figure}[htbp]
  \centering
  \includegraphics[width=0.45\textwidth,height=0.6\textwidth]{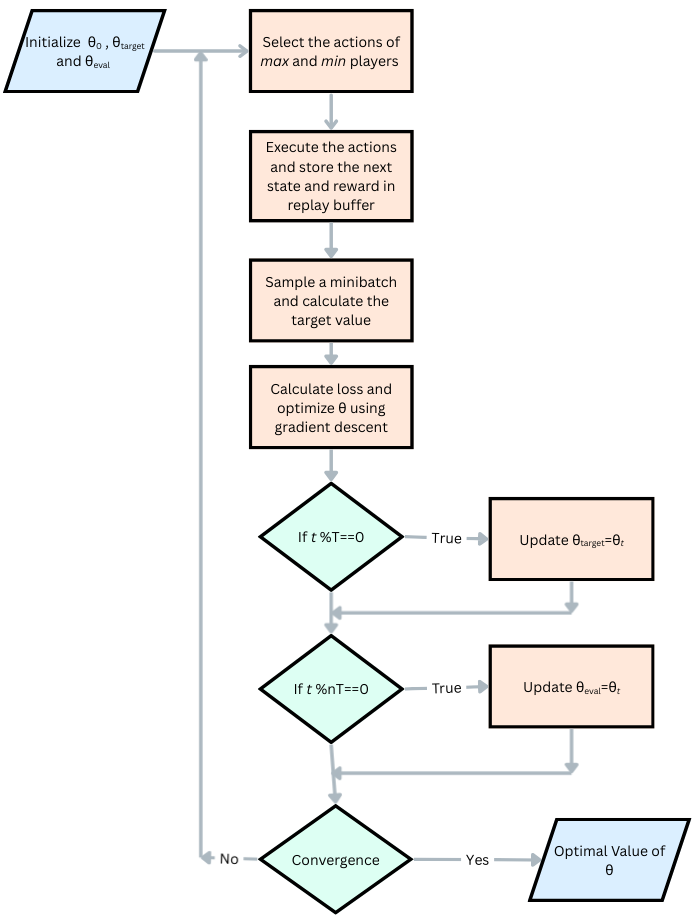} % file you uploaded
  \caption{Network parameters $\theta_0$, $\theta_{\text{eval}}$, and $\theta_{\text{target}}$. The current timestep is denoted by $t$; the target network parameters ($\theta_{\text{target}}$) are updated every T iterations, and the evaluation network parameters ($\theta_{\text{eval}}$) are updated every nT iterations, where n is the number of inner loops used for evaluation }
  \label{fig:myimage}
\end{figure}
\subsection{Proof of Lemma \ref{lemma:1}}
% \begin{proof}
\begin{align}
    \|\theta_{t+1}-\theta^*\| \nonumber &\leq \|\prod_{h=\tau}^{t}(1-\alpha_h)\theta_\tau +\sum_{h=\tau}^{t}\alpha_h\prod_{l=h+1}^{t}(1-\alpha_l)F(\theta_h) \nonumber + \sum_{h=\tau}^{t}\alpha_k\prod_{l=h+1}^{t}(1-\alpha_l)\mathcal{M}(h) - \theta^*\| \nonumber\\
    &\leq \|\prod_{h=\tau}^{t}(1-\alpha_h)\theta_\tau +\sum_{h=\tau}^{t}\alpha_h\prod_{l=h+1}^{t}(1-\alpha_l)F(\theta_h) -\theta^* \|  + \|\sum_{h=\tau}^{t}\alpha_h\prod_{l=h+1}^{t}(1-\alpha_l)\mathcal{M}(h)\|
\end{align}

Using :
\begin{align}
    \prod_{h=\tau}^{t}(1-\alpha_h) + \sum_{h=\tau}^{t}\alpha_h\prod_{l=h+1}^{t}(1-\alpha_l) = 1 
\end{align}
we have 
\begin{align}
    \|\theta_{t+1}-\theta^*\|  &\leq \prod_{h=\tau}^{t}(1-\alpha_h)\|\theta_\tau -\theta^*\| +\gamma'\sum_{h=\tau}^{t}\alpha_h\prod_{l=h+1}^{t}(1-\alpha_l)\|\theta_h-\theta^*\| +\|\sum_{h=\tau}^{t}\alpha_h\prod_{l=h+1}^{t}(1-\alpha_l)\mathcal{M}(h)\|
\end{align}
writing in compact form, we have :
\begin{align}
        \xi_{t+1}\leq \beta_{\tau-1,t} \xi_\tau + \gamma' \sum_{h=\tau}^{t}\widetilde{\beta}_{h,t}\xi_h + \|\sum_{h=\tau}^{t} \widetilde{\beta}_{h,t}\mathcal{M}(h)\|
\end{align}
where $\beta_{\tau-1,t}=\prod_{h=\tau}^t(1-\alpha_h)$, $\widetilde{\beta}_{h,t} =\alpha_h\prod_{l=h+1}^{t}(1-\alpha_l)$ and $a_h=\|\theta_h - \theta^*\|$
% \end{proof}

\subsection{Proof of Lemma \ref{lemma:2}}
% \begin{proof}
We have 
        \begin{align}    
            &(1-\alpha_t)=e^{\log(1-\frac{\mathcal{H}}{t+t_0})} \leq e^{-\frac{\mathcal{H}}{t+t_0}}\nonumber\\
            &\prod_{l=h+1}^{t}(1-\alpha_l)\leq e ^{-\sum_{h+1}^{t}\frac{\mathcal{H}}{t+t_0}} \leq e^{-\int_{h+1}^{t+1}\frac{\mathcal{H}}{y+t_0} dy} \nonumber\\ &\quad \quad \quad \quad \quad\leq e^{-\mathcal{H}\log \frac{t+1+t_0}{h+1+t_0}} = \bigg(\frac{h+1+t_0}{t+1+t_0}\bigg)^\mathcal{H}
        \end{align}
    Now we bound $\widetilde{\beta}_{h,t}^2$
        \begin{align}
        \label{eqn:beta}
            \widetilde{\beta}_{h,t}^2 &\leq \frac{\mathcal{H}^2}{(t+1+t_0)^{2\mathcal{H}}}\frac{(h+1+t_0)^{2\mathcal{H}}}{(h+t_0)^2} \leq \frac{2\mathcal{H}^2}{(t+1+t_0)^{2\mathcal{H}}}(h+t_0)^{2\mathcal{H}-2}
        \end{align}
    Eq \eqref{eqn:beta} comes from the fact that $(h+1+t_0)^{2\mathcal{H}}\leq 2(h+t_0)^{2\mathcal{H}}$, which is due to the condition on $t_0$ (Theorem\eqref{thm:1}). Subsequently we have, 
        \begin{align}
            \sum_{h=1}^{t}\widetilde{\beta}_{h,t}^2&\leq \frac{2\mathcal{H}^2}{(t+1+t_0)^{2\mathcal{H}}}\sum_{h=1}^{t}(h+t_0)^{2\mathcal{H}-2}\nonumber\\ &\leq \frac{2\mathcal{H}^2}{(t+1+t_0)^{2\mathcal{H}}}\int_{1}^{t+1}(y+t_0)^{2\mathcal{H}-2}\, dy\\
            &< \frac{2\mathcal{H}^2}{(t+1+t_0)^{2\mathcal{H}}}\frac{1}{2\mathcal{H}-1}(t+1+t_0)^{2\mathcal{H}-1}\\
            &< \frac{2\mathcal{H}}{t+1+t_0}
        \end{align}
    where in the last inequality we have used $2\mathcal{H}-1>\mathcal{H}$.
% \end{proof}
$\hfill \Box$

\subsection{Proof of Lemma \ref{lemma:3}}
% \begin{proof}
To prove Lemma \eqref{lemma:3}, we have used induction.
We consider the following summation
\begin{align}
c_t=\sum_{h=\tau}^t\widetilde{\beta}_{h,t}\frac{1}{(h+t_0)^\Gamma}
\end{align}
We have to show that $c_t \leq \frac{1}{\sqrt{\gamma'}(t+1+t_0)}$. 
For $t=\tau$, if $\alpha_\tau \leq \frac{1}{2}$, $(1+\frac{1}{t_0})^\Gamma \leq \frac{2}{\sqrt{\gamma'}}$
\begin{align}
c_\tau=\widetilde{\beta}_{\tau,\tau}\frac{1}{(\tau+t_0)^\Gamma}=\alpha_\tau\frac{1}{(\tau + t_0)^\Gamma}\leq \frac{1}{\sqrt{\gamma'}(t+1+t_0)^\Gamma}
\end{align}
Let the statement be true for $t-1$, then we need to prove it for $c_t$
\begin{align}
c_t&=\sum_{h=\tau}^{t}\widetilde{\beta}_{h,t}\frac{1}{(h+t_0)^\Gamma} \nonumber\\
&= \sum_{h=\tau}^{t-1}\alpha_h \prod_{l=h+1}^{t}(1-\alpha_l)\frac{1}{(h+t_0)^\Gamma} + \alpha_t \frac{1}{(t+t_0)^\Gamma}\nonumber\\
&=(1-\alpha_t)\sum_{h=\tau}^{t-1}\prod_{l=h+1}^{t-1}(1-\alpha_l)\frac{1}{(h+t_0)^\Gamma} + \frac{\alpha_t}{(t+t_0)^\Gamma}\nonumber\\
&=(1-\alpha_t)c_{t-1} + \frac{\alpha_t}{(t+t_0)^\Gamma} \nonumber \\ 
&\leq (1-\alpha_t)\frac{1}{\sqrt{\gamma'}(t+t_0)^\Gamma} + \frac{\alpha_t}{(t+t_0)^\Gamma}   
= [1-\alpha_t(1-\sqrt{\gamma'})]\frac{1}{\sqrt{\gamma'}(t+t_0)^\Gamma}
\end{align} 
put $\alpha_t=\frac{\mathcal{H}}{t+t_0}$

\begin{align}
    c_t &\leq \Bigg[1- \frac{\mathcal{H}(1-\sqrt{\gamma'})}{t+t_0}\Bigg]\frac{1}{\sqrt{\gamma'}(t+t_0)^\Gamma}\nonumber\\
    & = \Bigg[1- \frac{\mathcal{H}(1-\sqrt{\gamma'})}{t+t_0}\Bigg]\Bigg(\frac{t+1+t_0}{t+t_0}\Bigg)^\Gamma \frac{1}{\sqrt{\gamma'}(t+1+t_0)^\Gamma} \nonumber \\
    &= \Bigg[1- \frac{\mathcal{H}(1-\sqrt{\gamma'})}{t+t_0}\Bigg]\Bigg(1+\frac{1}{t+t_0}\Bigg)^\Gamma \frac{1}{\sqrt{\gamma'}(t+1+t_0)^\Gamma} \nonumber \\
    &\leq \Bigg[1- \frac{\mathcal{H}(1-\sqrt{\gamma'})}{t+t_0}\Bigg]\Bigg(1+\frac{\Gamma}{t+t_0}\Bigg) \frac{1}{\sqrt{\gamma'}(t+1+t_0)^\Gamma} \nonumber \\
    &\leq \Bigg[\Big(1+\frac{\Gamma}{t+t_0}\Big) - \frac{\mathcal{H}(1-\sqrt{\gamma'})}{t+t_0}\Big(1+\frac{\Gamma}{t+t_0}\Big)\Bigg] \nonumber
    \frac{1}{\sqrt{\gamma'}(t+1+t_0)^\Gamma}\nonumber\\
    &\leq \Bigg[\Big(1+\frac{\Gamma}{t+t_0}\Big) - \frac{\mathcal{H}(1-\sqrt{\gamma'})}{t+t_0}\Bigg]\frac{1}{\sqrt{\gamma'}(t+1+t_0)^\Gamma}\\
    &\leq \frac{1}{\sqrt{\gamma'}(t+1+t_0)^\Gamma}
\end{align}
% chnage here
Equation () comes from the fact that $\Gamma \leq1$ and $\mathcal{H}(1-\sqrt{\gamma'})\geq 1$. As $\frac{\Gamma}{t+t_0} - \frac{\mathcal{H}(1-\sqrt{\gamma'})}{t+t_0}>-1$ , we can write $\Bigg(1+\frac{\Gamma}{t+t_0} - \frac{\mathcal{H}(1-\sqrt{\gamma'})}{t+t_0}\Bigg) \leq e^{\frac{\Gamma}{t+t_0} - \frac{\mathcal{H}(1-\sqrt{\gamma'})}{t+t_0}}\leq 1$

\subsection{Proof of Assumption \ref{assump:1}}
% \begin{proof}
    \begin{align}
    &\| F(\theta_1) - F(\theta_2) \| = \|\mathbb{E}[\psi(s, a, b) \Bigg[ w \Big( \gamma \operatorname{val}(\psi(s', \cdot, \cdot)^{\top} \theta_1) \nonumber  - \gamma \operatorname{val}(\psi(s', \cdot, \cdot)^{\top} \theta_2) \Big) + (1 - w) \Big( \operatorname{val}(\psi(s, \cdot, \cdot)^{\top} \theta_1) - \nonumber\\
    & \quad \quad \quad \quad \operatorname{val}(\psi(s, \cdot, \cdot)^{\top} \theta_2) \Big) \Bigg] ]\|\nonumber \\
    &\leq \mathbb{E}[\|\psi(s, a, b) w \Big( \gamma \operatorname{val}(\psi(s', \cdot, \cdot)^{\top} \theta_1)- \gamma \operatorname{val}(\psi(s', \cdot, \cdot)^{\top} \theta_2) \Big) \nonumber  + \|\psi(s, a, b)(1 - w) \Big( \operatorname{val}(\psi(s, \cdot, \cdot)^{\top} \theta_1) - \operatorname{val}(\psi(s, \cdot, \cdot)^{\top} \theta_2) \Big)\| ]\nonumber\\
    &\leq \mathbb{E}[w \gamma| (  \operatorname{val}(\psi(s', \cdot, \cdot)^{\top} \theta_1)-  \operatorname{val}(\psi(s', \cdot, \cdot)^{\top} \theta_2) )| + (1-w)| ( \operatorname{val}(\psi(s, \cdot, \cdot)^{\top} \theta_1) - \operatorname{val}(\psi(s, \cdot, \cdot)^{\top} \theta_2) )|] 
\end{align}

\begin{table*}[!t]
    \centering
    \scalebox{0.8}{
    \begin{tabular}{|c|c|c|c|c|}
    \hline
       Algorithm & Multi-agent & reward settings & Fun. approx. &  Finite-time Analysis\\ \hline
       \cite{qu2020finite} & no & dis. reward & no & $O(t_{mix} \mid \mathcal{S} \mid \mid \mathcal{A} \mid \epsilon^{-2})$\\ \hline
        % RL-ARNE \cite{sahabandu2024rl}&yes, non-zero sum & avg. reward & no & $O(m \mathcal{A})$ \tablefootnote{m: the number of players, and memory complexity is $O(m \mathcal{S}\mathcal{A})$.}\\ \hline
          \cite{diddigi2022generalized} & yes, zero-sum & dis. reward & no & NA \tablefootnote{NA: Not Available}\\ \hline
           Proposed work & yes, zero-sum & dis. reward & yes & $O(\epsilon^{-2})$\\ \hline
    \end{tabular}}
    \caption{Comparative analysis of related works in terms of finite-time complexity}
    \label{tab_FA}
\end{table*}

Now, we will find the upper bound to $| (  \operatorname{val}(\psi(s, a_1, b_1)^{\top} \theta_1)-  \operatorname{val}(\psi(s, a_2, b_2)^{\top} \theta_2) )|$. 
    \begin{align}
        \min_{\substack{a_1}} \max_{\substack{b_1}}\psi(s,a_1,b_1)^{\top}\theta_1 -\min_{\substack{a_2}} \max_{\substack{b_2}}\psi(s,a_2,b_2)^{\top}\theta_2 \nonumber
        & \leq -\min_{\substack{b'}}\{-\max_{\substack{a_1}}\psi(s,a_1,b')^{\top}\theta_1 + \max_{\substack{a_2}}\psi(s,a_2,b')^{\top}\theta_2\} \nonumber\\
        & =\max_{\substack{b'}}\{\max_{\substack{a_1}}\psi(s,a_1,b')^{\top}\theta_1 - \max_{\substack{a_2}}\psi(s,a_2,b')^{\top}\theta_2\}\nonumber\\
        & \leq\max_{\substack{b',a'}}\psi(s,a',b')^{\top}(\theta_1-\theta_2)\nonumber\\
        & \leq\max_{\substack{b',a'}}|\sum_{i=1}^{d}\psi_{i}(s,a',b')(\theta_{1i}-\theta_{2i})| 
    \end{align}
    
    \begin{align}
        \min_{\substack{a_1}} \max_{\substack{b_1}}\psi(s,a_1,b_1)^{\top}\theta_1 -\min_{\substack{a_2}} \max_{\substack{b_2}}\psi(s,a_2,b_2)^{\top}\theta_2\nonumber
        &\geq\min_{\substack{b'}}\{\max_{\substack{a_1}}\psi(s,a_1,b')^{\top}\theta_1 - \max_{\substack{a_2}}\psi(s,a_2,b')^{\top}\theta_2\} \nonumber\\ & \geq\min_{\substack{b'}}\{\max_{\substack{a_1}}\psi(s,a_1,b')^{\top}\theta_1 - \max_{\substack{a_2}}\psi(s,a_2,b')^{\top}\theta_2\}\nonumber\\
        &\geq\min_{\substack{b'}}\{-\max_{\substack{a'}}\psi(s,a',b')^{\top}(\theta_2 -\theta_1)\}\nonumber \\
        &\geq-\max_{\substack{b',a'}}\psi(s,a',b')^{\top}(\theta_2-\theta_1)\nonumber\\
        & \geq-\max_{\substack{b',a'}}|\sum_{i=1}^{d}\psi_{i}(s,a',b')(\theta_{1i}-\theta_{2i})|
    \end{align}
Therefore,
    \begin{align}
        |\min_{\substack{a_1}} \max_{\substack{b_1}}\psi(s,a_1,b_1)^{\top}\theta_1 -\min_{\substack{a_2}} \max_{\substack{b_2}}\psi(s,a_2,b_2)^{\top}\theta_2|\nonumber
        &\leq \max_{\substack{b',a'}}|\psi(s,a',b')^{\top}(\theta_1-\theta_2)|\nonumber \\
        & \leq\max_{\substack{b',a'}}\|\psi(s,a',b')\|\|(\theta_1-\theta_2)\|\nonumber\\
        & \leq\|\theta_1-\theta_2\|
    \end{align}
Thus we have,
    \begin{align}
        \| F(\theta_1) - F(\theta_2) \| \leq (w\gamma+1-w)\|\theta_1-\theta_2\|
    \end{align}
% \end{proof}
$\hfill \Box$
\subsection{System Specifications and Larger size plots of experimental results}

\textbf{System Specifications:} All experiments are conducted on an NVIDIA RTX A6000 GPU with 128 GB of RAM.\\

\textbf{Larger versions of the plots:} 

\begin{figure*}[htb]
\centering
\vspace{-30pt}
  \subfloat[]{
  \hspace{-20pt}
	\begin{minipage}[c][1\width]{
	   0.5\textwidth}
	   \centering
        \includegraphics[width=1\textwidth]{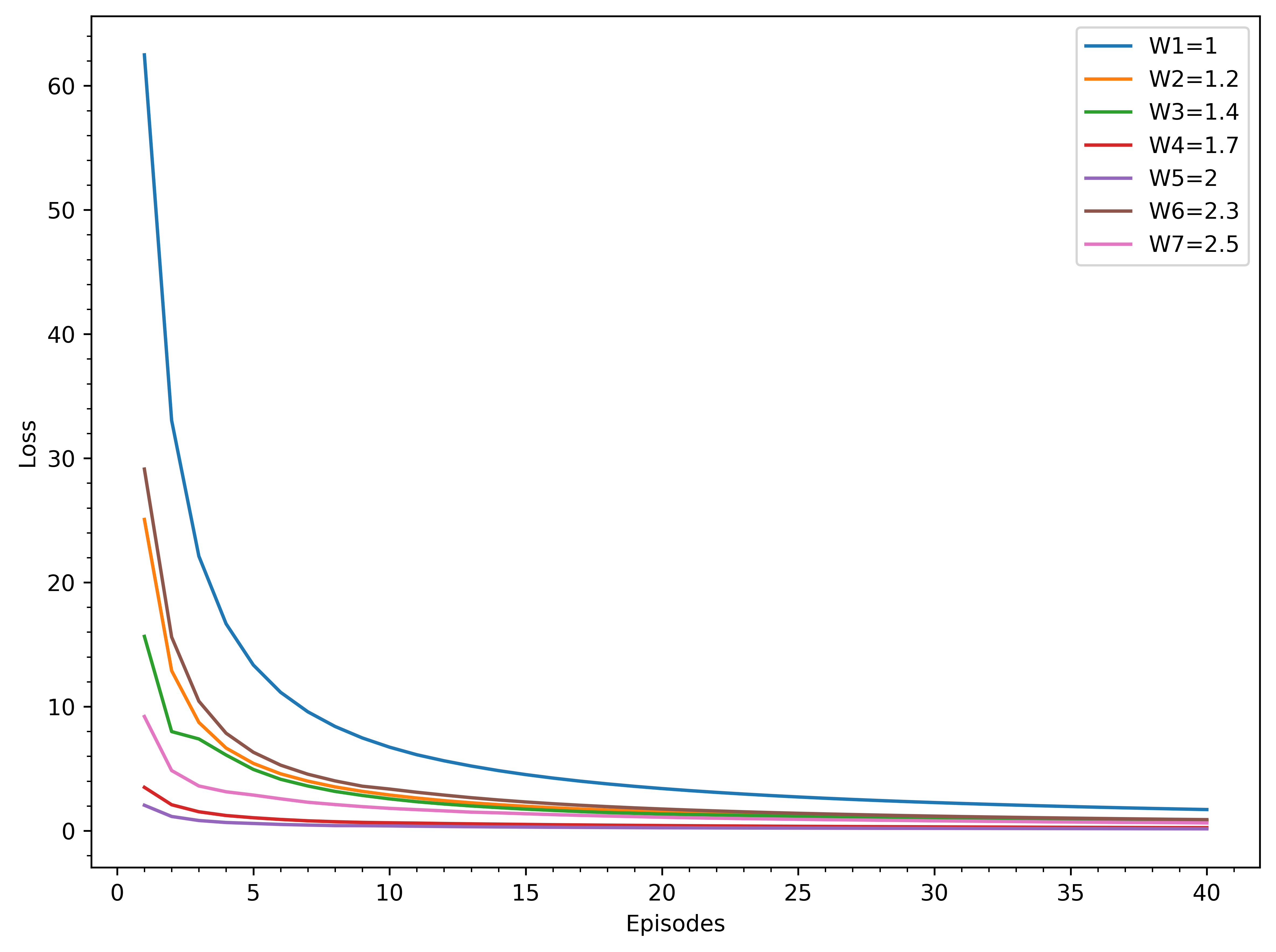}
        \vspace{-35pt}
	\end{minipage}}	
  \subfloat[]{
	\begin{minipage}[c][1\width]{
	   0.5\textwidth}
        % \hspace{-20pt}
	   \centering
	   \includegraphics[width=1\textwidth]{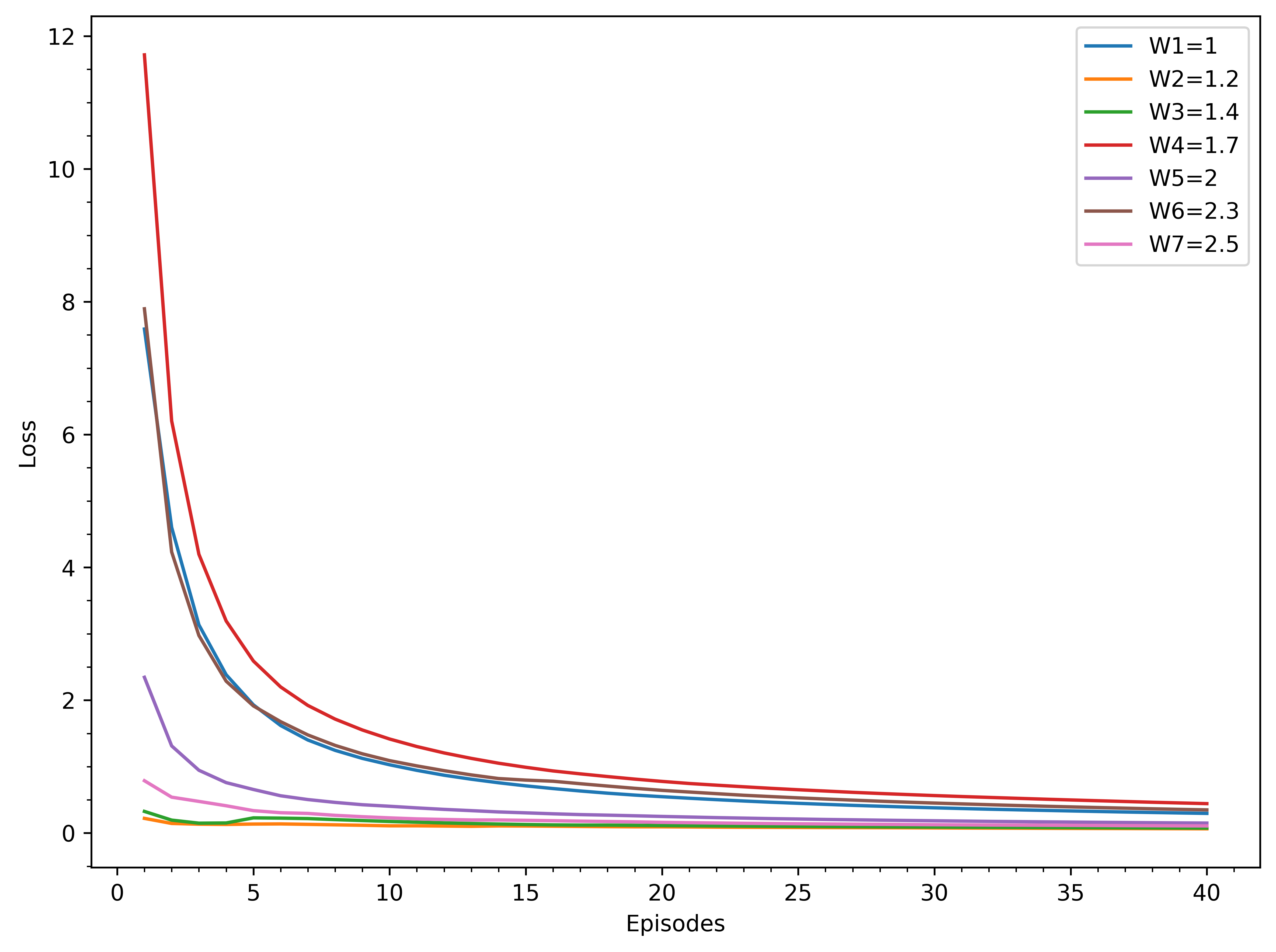}
    \vspace{-35pt}
	\end{minipage}}
 % \vspace{-10pt}
\caption{Loss on Guard-Invader environment with  $(a)$ $49$, and $(b)$ $121$ states, respectively.}
\label{fig_GI_loss_appendix}
\end{figure*}

\begin{figure*}[htb]
\centering
\vspace{-30pt}
  \subfloat[]{
  \hspace{-20pt}
	\begin{minipage}[c][1\width]{
	   0.5\textwidth}
	   \centering
        \includegraphics[width=1\textwidth]{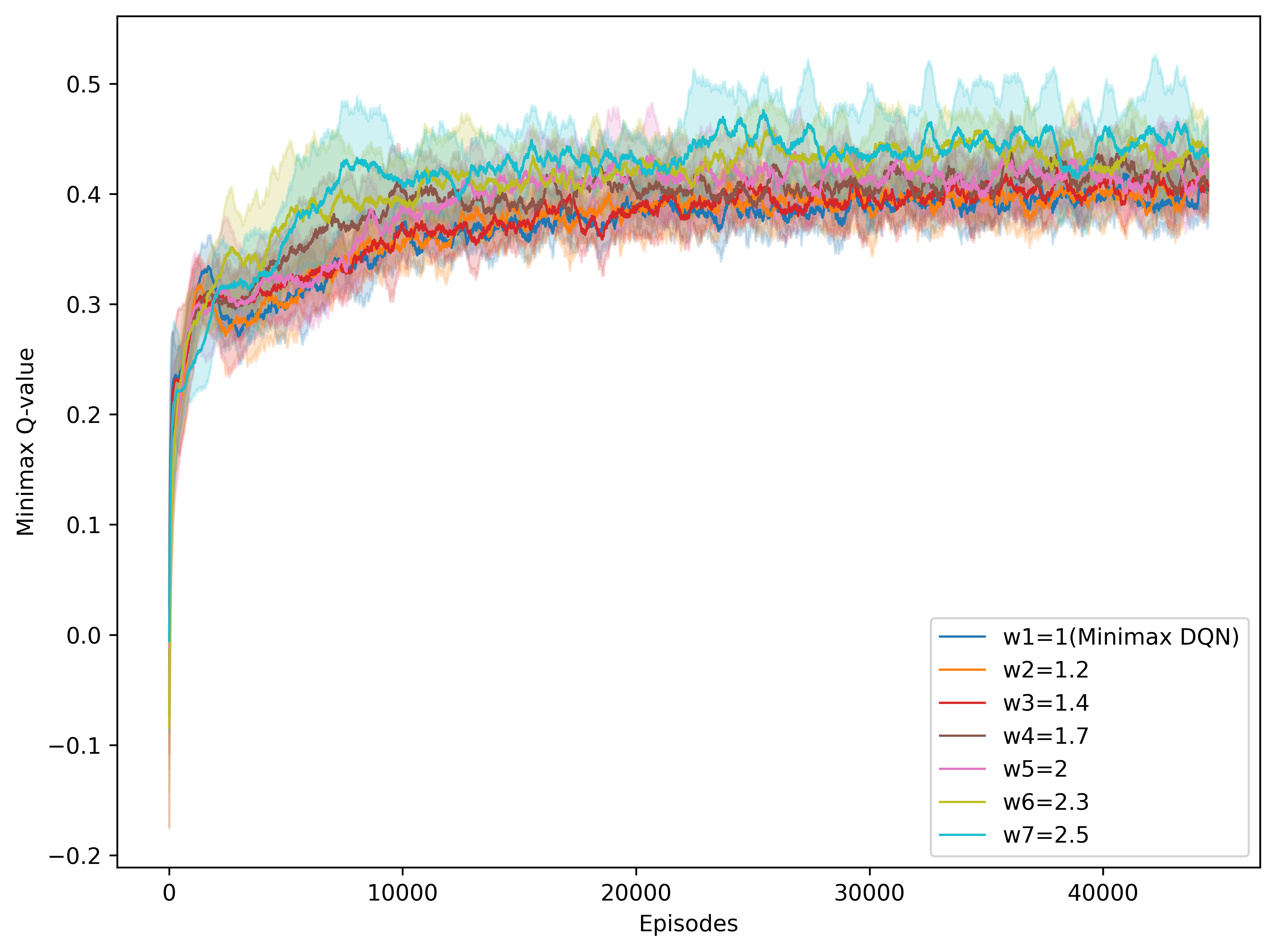}
        \vspace{-35pt}
	\end{minipage}}	
  \subfloat[]{
	\begin{minipage}[c][1\width]{
	   0.5\textwidth}
        % \hspace{-20pt}
	   \centering
	   \includegraphics[width=1\textwidth]{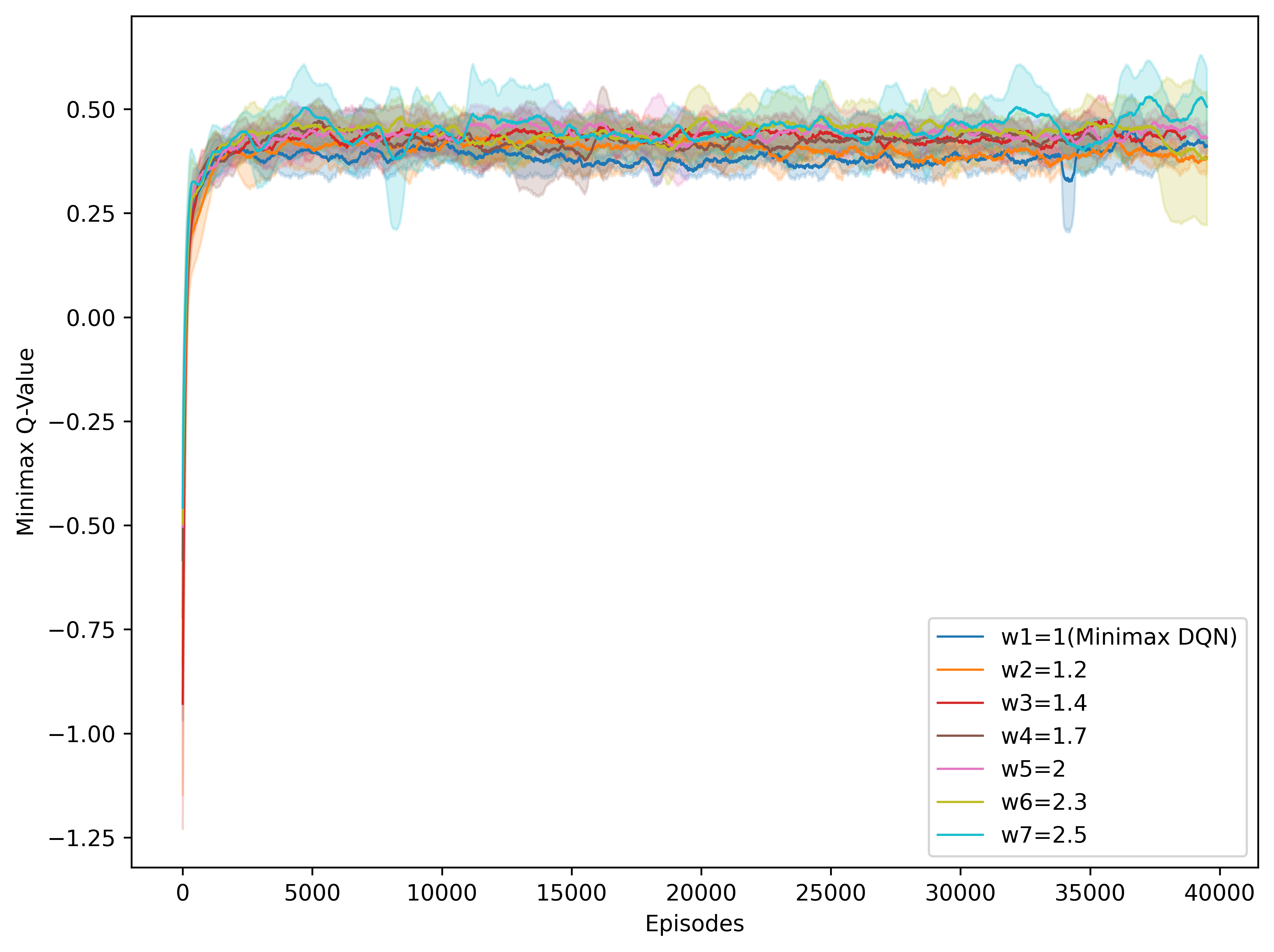}
    \vspace{-35pt}
	\end{minipage}}
 % \vspace{-10pt}
\caption{Mimimax Q-value on Guard-Invader environment with  $(a)$ $49$, and $(b)$ $121$ states, respectively.}
\label{fig_GI_Q_val_appendix}
\end{figure*}

%%% Soccer env
\begin{figure*}[htb]
\centering
\vspace{-30pt}
  \subfloat[]{
  \hspace{-20pt}
	\begin{minipage}[c][1\width]{
	   0.5\textwidth}
	   \centering
        \includegraphics[width=1\textwidth]{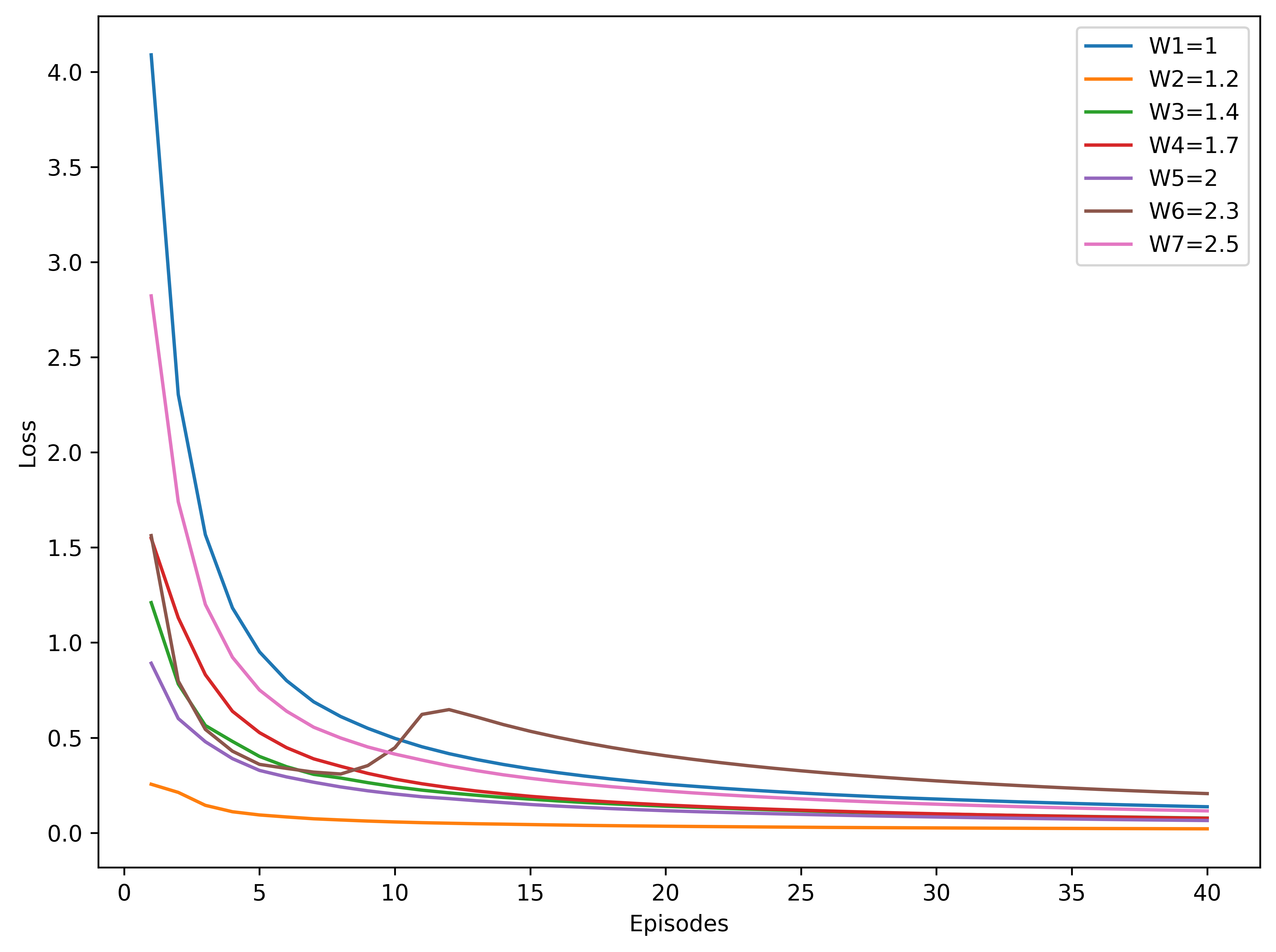}
        \vspace{-35pt}
	\end{minipage}}	
  \subfloat[]{
	\begin{minipage}[c][1\width]{
	   0.5\textwidth}
        % \hspace{-20pt}
	   \centering
	   \includegraphics[width=1\textwidth]{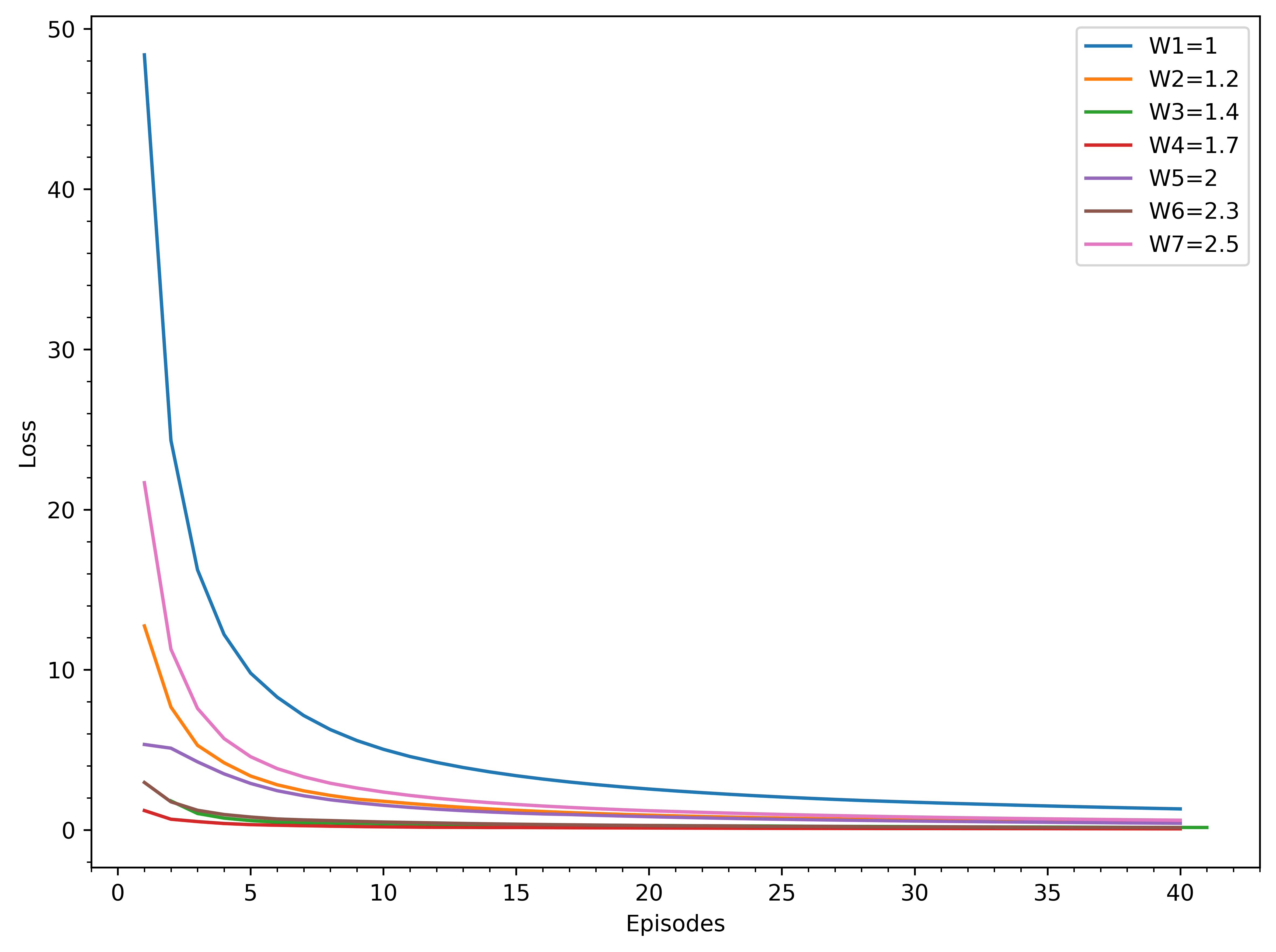}
    \vspace{-35pt}
	\end{minipage}}
 % \vspace{-10pt}
\caption{Loss on Soccer environment with  $(a)$ $49$, and $(b)$ $121$ states, respectively.}
\label{fig_Soccer_loss_appendix}
\end{figure*}

\begin{figure*}[htb]
\centering
\vspace{-30pt}
  \subfloat[]{
  \hspace{-20pt}
	\begin{minipage}[c][1\width]{
	   0.5\textwidth}
	   \centering
        \includegraphics[width=1\textwidth]{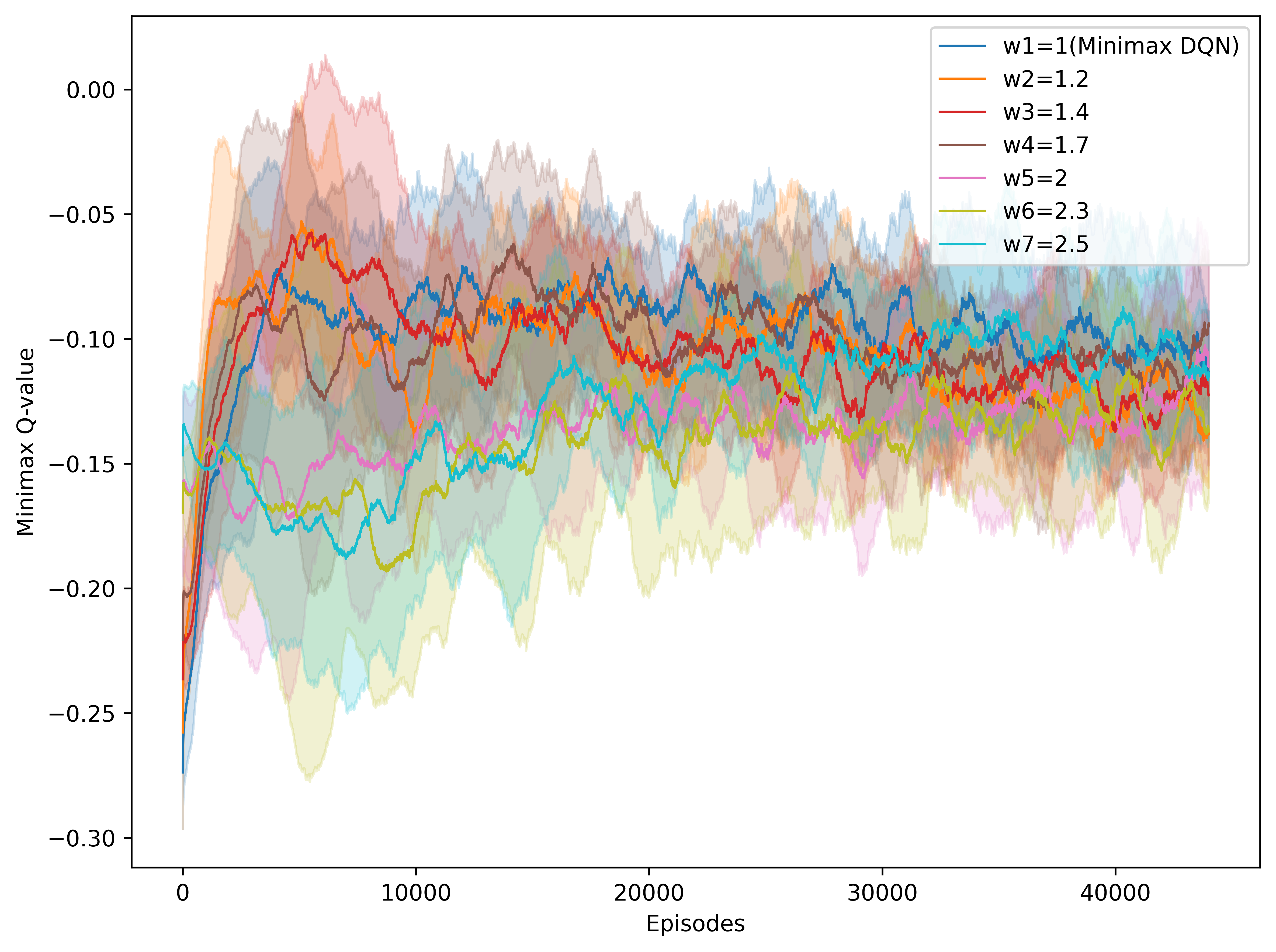}
        \vspace{-35pt}
	\end{minipage}}	
  \subfloat[]{
	\begin{minipage}[c][1\width]{
	   0.5\textwidth}
        % \hspace{-20pt}
	   \centering
	   \includegraphics[width=1\textwidth]{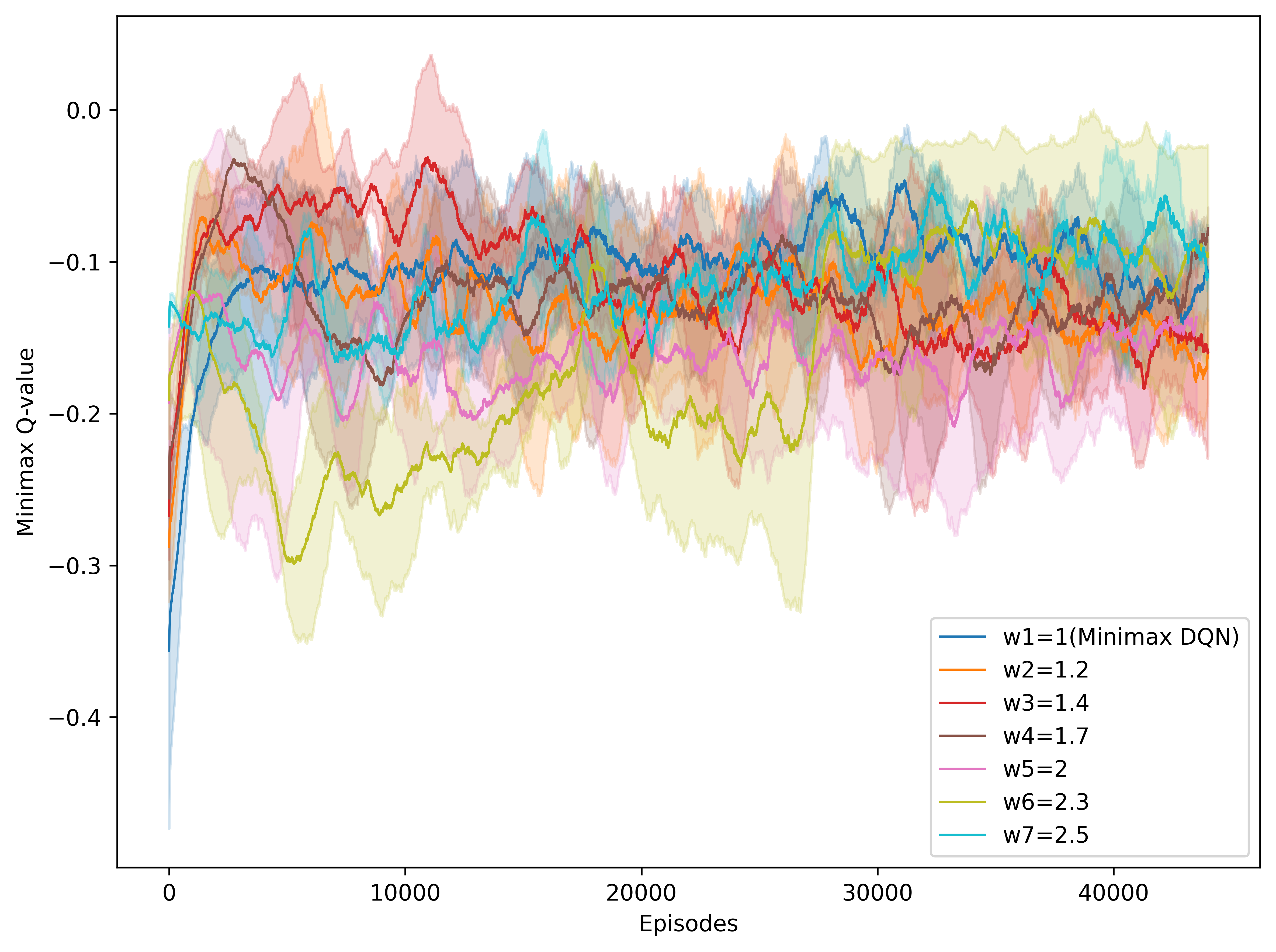}
    \vspace{-35pt}
	\end{minipage}}
 % \vspace{-10pt}
\caption{Mimimax Q-value on Soccer environment with  $(a)$ $49$, and $(b)$ $121$ states, respectively.}
\label{fig_Soccer_Q_val_appendix}
\end{figure*}
\end{document}